\documentclass[journal,twoside,web]{ieeecolor}
\usepackage{lcsys}
\usepackage{generic}
\usepackage{cite}
\usepackage{amsmath,amssymb,amsfonts}
\usepackage{algorithmic}
\usepackage{graphicx}
\usepackage{algorithm,algorithmic}
\usepackage{hyperref}
\usepackage{textcomp}
\usepackage{booktabs}
\usepackage{comment}
\usepackage{mathrsfs}
\usepackage[T1]{fontenc}
\usepackage{oldgerm}
\usepackage{scalerel,stackengine}
\usepackage{stmaryrd}

\newtheorem{theorem}{Theorem}[section]
\newtheorem{definition}{Definition}[section]

\newtheorem{proposition}{Proposition}[section]
\newtheorem{lemma}{Lemma}[section]
\newtheorem{corollary}{Corollary}[section]
\newtheorem{remark}{Remark}[section]
\newtheorem{assumption}{Assumption}[section]
\newtheorem{claim}{Claim}[section]

\newcommand{\js}[1]{\textcolor{blue}{(JS: #1)}}

\def\BibTeX{{\rm B\kern-.05em{\sc i\kern-.025em b}\kern-.08em
    T\kern-.1667em\lower.7ex\hbox{E}\kern-.125emX}}
\begin{document}
\title{The Attention Within:  Consensus Dynamics in Selective State Space Models}
\author{
    João Pedro Silvestre$^{1}$, \'Alvaro Rodr\'iguez Abella$^{2}$ and Paulo Tabuada$^{1}$
    \thanks{This work was partially supported by the NSF award 2502536 and the Air Force Office of Scientific Research under the Multidisciplinary University Research Initiative grant Hybrid Dynamics - Deconstruction and Aggregation (HyDDRA). J.P.S was partially supported by the
PhD fellowship 2023.01843.BD from the Fundação para a Ciência e a Tecnologia (FCT), Portugal. A.R.A. was partially supported by grant PID2024-156578NB-I00 funded by MICIU/AEI/10.13039/501100011033/FEDER, EU.}\thanks{$^{1}$João Pedro Silvestre and Paulo Tabuada are with the Electrical and Computer Engineering Department, University of California at Los Angeles, Los Angeles, CA 90095 USA (e-mail: {\tt\small \{joaosilvestre, tabuada\}@ucla.edu}).}
\thanks{$^{2}$\'Alvaro Rodr\'iguez Abella is with the Department of Applied Mathematics, Comillas Pontifical University, Madrid, 28015 - Madrid, Spain (e-mail: {\tt\small arabella@comillas.edu}).}
}

\maketitle

\begin{abstract}
Selective state space models (SSMs) have recently emerged as a compelling alternative to transformers, combining competitive performance with substantially improved inference efficiency. At each SSM layer, a sequence of hidden states are propagated by a recurrence, mixing information of different tokens. Despite using a different mechanism, this mixing plays a role analogous to attention in transformers. In fact, recent works have shown that the two architectures may be closer than they first appear, as this recurrence admits a formulation akin to linear attention. In transformers, attention is known to drive the tokens to cluster, i.e., to reach consensus, collapsing in the limit to a single direction. Thus, we ask: does the recurrence at the core of SSMs drive the tokens to consensus, as attention does in transformers?

To answer this question, we take a dynamical systems perspective on SSMs, modeling the evolution of tokens across layers as an ordinary differential equation. By exploiting input-to-state stability arguments, we establish local exponential stability of the consensus equilibria and characterize their domain of attraction for time-varying weight matrices, a setting not addressed by previous results. We thereby show that the resemblance between SSMs and transformers does run deeper: the recurrence at the core of SSMs aggregates tokens just as attention does. Numerical experiments on a pretrained Mamba-2 model point to the output gate as the component that regulates the extent of this consensus, preventing the tokens from reaching it in full.
\end{abstract}


\maketitle

\section{Introduction}\label{sec:intro}

In recent years, large language models (LLMs) have seen widespread adoption across a rapidly expanding range of tasks~\cite{zhao2026survey}. Chief among these models is the transformer~\cite{vaswani2017attention}, which has emerged as the dominant architectural paradigm: it underpins the foundation models in widest use today, including popular models such as ChatGPT~\cite{achiam2023gpt}. The reach of the architecture, however, extends well beyond language. Attention, the mechanism at its core, first introduced for neural machine translation \cite{bahdanau2014neural}, has since proved to be a powerful mechanism, allowing the introduction of transformers into domains such as vision \cite{dosovitskiy2020image} and protein structure prediction \cite{jumper2021highly}.

However, transformers are not without shortcomings. The first is computational: the cost of attention grows quadratically with the sequence length~\cite{tay2022efficient}. The second emerges with depth. As models grow deeper, now reaching hundreds of layers~\cite{waleffe2024empirical, grattafiori2024llama}, the returns diminish and the expressive power of the network saturates beyond a certain point~\cite{levine2020limits}, while the tokens grow alike as more and more layers are traversed~\cite{dong2021attention, noci2022signal}. What to make of this last effect is still contested. Some works see it as a defect~\cite{dong2021attention}, since tokens that become indistinguishable can no longer carry distinct information. Others see it as the very mechanism by which the model groups related tokens~\cite{geshkovski2023emergence}, and it has been used directly to solve language tasks, by clustering the tokens of a sentence around the ones that carry most of its meaning~\cite{alcalde2024clustering}.

Selective state space models, first introduced in~\cite{gu2024mamba, dao2024transformers}, were designed to reduce the quadratic cost of attention at inference time. In doing so, they retain the benefits of recurrent neural networks~\cite{hochreiter1997long} while avoiding their classical shortcomings, such as vanishing gradients and the lack of parallelism during training.

The differences between SSMs and transformers, however, may be more subtle than they first appear. Recent works have shown that the core of an SSM admits a formulation akin to linear attention~\cite{dao2024transformers, ali2025hidden}, exposing intrinsic similarities between the two architectures. However, that similarity is structural, and does not by itself determine how the tokens behave across layers. A natural question thus arises: do the models share fundamental dynamical properties?

Several works have shown that
SSMs already mirror transformers in some of their expressiveness
barriers~\cite{merrill2024illusion}; empirical studies suggest that the parallel
may extend to the consensus phenomenon, where SSM tokens can cluster and become
increasingly indistinguishable across layers~\cite{wang2025oversmoothing,
skean2025comparative}. For transformers, the phenomenon is by now well documented: as layers accumulate, the tokens cluster together and drift toward a common direction~\cite{geshkovski2023mathematical}, an effect that has been analysed with time-varying weights, multiple heads, and in the autoregressive setting~\cite{rodriguez2024asymptotic, rodriguez2025consensus}. For SSMs, the comparable theory is far more limited, confined to time-invariant parameters and no normalization~\cite{vo2025demystifying,joseph2025lambda}, a setting too narrow to capture the models used in practice.

In this paper, we show that, in fact, the resemblance does run deeper: it extends to the dynamical evolution of tokens across layers. We do so by taking a dynamical systems perspective on selective SSMs: we model the evolution of tokens across the layers of the Mamba-2 model as an ordinary differential equation, and exploit its causal cascade structure to establish consensus through an input-to-state stability (ISS) argument~\cite{sontag2008input, khalil2002nonlinear}. To the best of our knowledge, these are the first such results for SSMs with time-varying weight matrices, the regime that faithfully reflects how their parameters vary across layers. We then analyse the Mamba-2 model experimentally, confirming that its core clusters the tokens and identifying the output gate as the component that prevents full consensus.

Our contributions are threefold:

\begin{enumerate}
\item We derive a continuous-time model of token evolution in Mamba-2 capturing multi-dimensional tokens, time-varying weights, and layer normalization.
\item We prove local exponential stability of the consensus equilibria, under persistence of excitation.
\item We characterize the domain of attraction of the consensus equilibria.
\end{enumerate}

\subsection*{Notations}

Let $r,s,\ell\in\mathbb N=\{1,2,\hdots\}$. The space of $r\times s$ real matrices is denoted by $\mathbb R^{r\times s}$. In particular, $\mathbb I_r\in\mathbb R^{r\times r}$ denotes the identity matrix. The transpose and Frobenius norm of a matrix $A\in\mathbb R^{r\times s}$ are denoted by $A^\top$ and $\|A\|$, respectively. Given $a_i\in \mathbb{R}$, $1\leq i\leq r$, the diagonal matrix with entries $a_1,\hdots,a_r$ is denoted by $\operatorname{diag}(a_1, \ldots, a_r) \in \mathbb{R}^{r \times r}$. Points in the Euclidean space $\mathbb{R}^r$ are regarded as column vectors and denoted by $x=(x^1,\dots,x^r)\in\mathbb R^r\equiv\mathbb R^{r\times 1}$. Tuples of $\ell$ points are denoted by $X=(x_1,\hdots,x_\ell)\in(\mathbb R^r)^\ell$. Open intervals are denoted by $]a,b[\,$, while closed intervals are denoted by $[a,b]$. In particular, we denote $\mathbb R_0^+=[0,\infty[$ and $\mathbb R^+=\,]0,\infty[\,$. The tangent space of a smooth manifold $M$ at $p\in M$ and its elements are denoted by $T_pM$ and $X_p\in T_p M$, respectively. Given another smooth manifold $N$ and a smooth map \mbox{$\phi:M\to N$}, \textit{i.e.}, \mbox{$\phi\in C^\infty(M,N)$}, the corresponding tangent map is denoted by \mbox{$T\phi:TM\to TN$}.

\section{Dynamics of selective state space models}

In this section, we introduce the Mamba-2 model and derive a continuous-time
approximation of its dynamics. Our objective is to derive a model
that directly relates the input and output of each layer, allowing us to then analyse the model as a dynamical system. In these models the input is typically considered to be a token, \textit{i.e.}, a
numerical representation of a word or sub-word, and the output sequence is
used to predict the next word or sub-word in the sequence. 

While input-output
equations for the Mamba-2 have been derived in~\cite{ali2025hidden} in
discrete time and in~\cite{vo2025demystifying} in continuous time, our
formulation differs from both by incorporating layer normalization, which
projects the tokens onto the unit sphere
$\mathbb{S}^{n-1}$ after each layer update. This constraint is natural due to the use of the RMSNorm~\cite{zhang2019root}, the normalization used
in both Mamba and Mamba-2, which projects each token to a sphere.

\subsection{Configuration space}

Let $n\in\mathbb N$ and consider the Euclidean inner product on $\mathbb R^n$, \textit{i.e.}, $\langle x_1,x_2\rangle=x_1^\top\,x_2$ for $x_1,x_2\in\mathbb R^n$. The corresponding norm is denoted by $|x|=\langle x,x\rangle^{1/2}$. The points of $\mathbb R^n$ of unit norm define the $(n-1)$-dimensional sphere:
$$\mathbb S^{n-1} =\{y\in\mathbb R^n\mid\langle y,y\rangle=1\}.
$$
As we consider a model consisting of $\ell$ tokens, the resulting state space is the Cartesian product of $\ell$ copies of the $n$-sphere:
\begin{align*}
(\mathbb S^{n-1})^\ell=\underbrace{\mathbb S^{n-1} \times{\hdots}\times\mathbb S^{n-1} }_{\ell\text{-times}}.
\end{align*}
Similarly, consider the sphere projection:
\begin{equation*}
\pi:\mathbb R^n-\{0\}\to\mathbb S^{n-1},\quad x\mapsto\pi(x)=x\,|x|^{-1}.
\end{equation*}
Its tangent map at each \mbox{$x\in\mathbb R^n-\{0\}$}, $T_x\pi:T_x(\mathbb R^n-\{0\})\to T_{\pi(x)}\mathbb S^{n-1}$, is given by:
\begin{equation*}
T_x\pi\cdot X_x =|x|^{-1}\left(\mathbb I_n-x \,x^\top\,|x|^{-2}\right)\cdot X_{x},
\end{equation*}
for each $X_x\in T_x(\mathbb R^n-\{0\})$. In particular, for $z\in\mathbb S^{n-1}$, it reads $T_z\pi\cdot X_z=\left(\mathbb I_n-z\,z^\top\right)\cdot X_z$.

\subsection{The Mamba-2 model}\label{sec:model}

Analogously to transformers, the Mamba-2 model can be described as a sequence-to-sequence map: given $n\in\mathbb N$, the model takes
an input sequence of $\ell\in\mathbb N$ tokens:
\begin{align*}
Z(1) = (z_1(1), \dots, z_\ell(1)) \in (\mathbb{S}^{n-1})^\ell,
\end{align*}
and produces an output sequence:
\begin{align*}
Z(\kappa) = (z_1(\kappa), \dots, z_\ell(\kappa)) \in (\mathbb{S}^{n-1})^\ell,
\end{align*}
where $\kappa\in\mathbb N$ is the depth of the model, \textit{i.e.}, the number of layers, and $z_i(k)$ denotes the $i$-th token at layer $k\in\{1,\hdots,\kappa\}$. Moreover, the output of each layer is dependent on its input and the layer index $k$, described as $Z(k+1)=f(k,Z(k))$, $1\leq k\leq \kappa-1$. In this work, we examine the asymptotic behavior of the model, \textit{i.e.}, its evolution as the number of layers increases indefinitely $\kappa\to\infty$.

Since the Mamba-2 model operates with a separate set of parameters for each of the
$n$ components of a token, we use $\mu \in \{1, \dots, n\}$ to index
quantities that vary across components\footnote{In the machine learning
literature, each component $\mu$ is commonly referred to as a
\emph{channel}.}. Thus, the scalar
$z_i^\mu(k)$ denotes the $\mu$-th entry of the $i$-th token at layer $k$, where $\mu\in\{1,\hdots,n\}$, $i\in\{1,\hdots,\ell\}$ and $k\in\mathbb N$.

Similarly to older models, such as Recurrent Neural Networks, the model maintains a hidden state indexed by each token component, \textit{i.e.}, a vector $h_i^\mu(k) \in \mathbb{R}^m$ that acts as a
compressed memory of the preceding $i-1$ tokens. This hidden state $h_i^\mu(k)$ is computed from $z_i(k)$ and $h_{i-1}^\mu(k)$ by the recurrence relation~\eqref{SSM:recurrence_hidden}, which is a function of the following input-dependent matrices:
\begin{align*}
& A_i^\mu(k) = \exp(-\alpha^\mu(k)\,\Delta^\mu_i(k))\,\mathbb{I}_m\in\mathbb R^{m\times m},\\
& B_i^\mu(k) = \Delta^\mu_i(k)\, S_B(k)\, z_i(k) \in \mathbb{R}^m,
\end{align*}
where $\alpha^\mu(k) \in \mathbb{R}$ is the learnable decay rate, and $\Delta_i^\mu(k) =
\operatorname{softplus}(W_\Delta^\mu(k)\, z_i(k) + b_\Delta^\mu(k))$, with $S_B(k)\in \mathbb{R}^{m \times n}$,
$W_\Delta^\mu(k) \in \mathbb{R}^{1 \times n}$ and
$b_\Delta^\mu(k) \in \mathbb{R}$ consisting of trainable weights. Recall that $\operatorname{softplus}(a)=\ln(\exp(a)+1)$ for each $a\in\mathbb R$.

The recurrence equations for the Mamba-2 consist of two coupled updates
operating along different axes: one propagates the hidden state $h^\mu_i$ along the
sequence (indexed by $i$), and the other propagates the token $z_i$ across layers
(indexed by $k$). The hidden state evolves independently along each
component $\mu$:
\begin{align}\label{SSM:recurrence_hidden}
\boxed{
h_i^\mu(k) = A_i^\mu(k)\, h_{i-1}^\mu(k) + B_i^\mu(k)\, z_i^\mu(k),
}
\end{align}
for each  $\mu\in\{1,\hdots,n\}$, $i\in\{1,\hdots,\ell\}$ and $k\in\mathbb N$, where $h_0^\mu(k) = 0$. The
token update then assembles all components and projects onto the sphere:
\begin{align}\label{SSM:recurrence_update}
\boxed{
z_i(k+1) = \pi\,\big(z_i(k) + \tau\, r_i(k)\big),
}
\end{align}
where $r_i(k)\in \mathbb{R}^n$ is given by\ $r_i^\mu(k) = z_i^\top(k)\,S_C^\top(k)\, h_i^\mu(k)$, $\mu\in\{1,\hdots,n\}$, with $S_C(k)\in\mathbb R^{m\times n}$ consisting of trainable weights, and $\tau\in\mathbb R^+$ being a small scale parameter.

Note that the first equation
propagates the hidden state $h_i^\mu(k)$ along the sequence: for a fixed layer
$k$, it accumulates a compressed representation of tokens $1$ through $i$,
with $A_i^\mu(k)$ controlling how past information decays and $B_i^\mu(k)$
injecting information about the current token. The second equation propagates the tokens across
layers: it updates $z_i(k)$ to $z_i(k+1)$ by reading out from the hidden
state via the output matrix $C_i(k)=z_i^\top(k)\,S_C^\top(k)$, adding the skip connection, and normalizing onto
$\mathbb{S}^{n-1}$. The parameter $\tau \in \mathbb{R}^+$ is a small
training weight that is usually absorbed into the output matrix, as
both are learned during training. Here we keep $\tau$ explicit, as it plays
the role of a step size in the derivation of the continuous-time model in section~\ref{sec:contmodel}.

{
The model presented so far excludes two nonlinearities of the Mamba-2 architecture, which we now make explicit. The first is the output gate. In the architecture, the readout $r_i(k)$ in~\eqref{SSM:recurrence_update} is multiplied elementwise by a gate $g(z_i(k)) \in \mathbb{R}^n$ before the skip connection and normalization, so that the update reads
$
z_i(k+1) = \pi\big(z_i(k) + \tau\, g(z_i(k)) \odot r_i(k)\big),
$
where $\odot$ denotes elementwise multiplication. Similarly to~\cite{vo2025demystifying}, we exclude the gate in order to isolate the recurrence, which is the component through which the tokens interact and the one that plays the role of attention. As we show in Section~\ref{Mamba}, the gate is the main component that attenuates the consensus induced by the recurrence.

The second nonlinearity concerns the matrices $B_i^\mu(k)$ and $C_i(k)$ defined in Section~\ref{sec:model}. Variants of the architecture differ in this respect, and one common choice, present in the Mamba-2 model considered in Section~\ref{Mamba}, is to pass the branches producing $B$ and $C$ through a SiLU nonlinearity. In this case, $S_B(k)z_i(k)$ and $S_C(k)z_i(k)$ are replaced by $\varsigma(S_B(k)z_i(k))$ and $\varsigma(S_C(k)z_i(k))$, respectively, where $\varsigma$ denotes the elementwise SiLU, i.e.,
$
\varsigma(x)^\mu = x^\mu(1+e^{-x^\mu})^{-1},
$
for each $\mu\in\{1,\dots,n\}$. Since the particular nonlinearity used depends on the specific architecture considered, we exclude it from the model and address its effect separately in Lemma~\ref{lemma:silu-bias}.
}

\subsection{Discrete time input-output model}
The recurrence presented in the previous section is written
per-component and depends on the hidden state, which
prevents a direct analysis of the full token dynamics. To
obtain a closed-form expression for the token update, we
unroll the hidden-state recurrence. To this end, we require some prior definitions that simplify the notation.

\subsubsection*{Notations}

For each $i,j\in\{1,\hdots,\ell\}$ with $j\leq i$ and $k\in\mathbb N$, let
$D_{ij}(k) \in \mathbb{R}^{n \times n}$ be the diagonal matrix:
\begin{align*}
D_{ij}(k) = \operatorname{diag}\!\left(
\lambda_{ij}^1(k), \dots,
\lambda_{ij}^n(k)\right),
\end{align*}
where:
\begin{align*}
\triangleright\quad
& \lambda_{ij}^\mu(k) =
\Delta_j^\mu(k)\,\exp(d_{ij}^\mu(k)),\\
\triangleright\quad
& d_{ij}^\mu(k) =
-(1-\delta_{ij})\,\alpha^\mu(k)
\sum_{l=j+1}^{i} \Delta_l^\mu(k),
\end{align*}
with $\delta_{ij}$ denoting the Kronecker's delta. The
eigenvalues of $D_{ij}(k)$ are
$\lambda_{ij}^\mu(k) \in \mathbb{R}^+$,
$1 \le \mu \le n$, with associated unit eigenvectors:
\begin{align*}
\mathfrak{e}^\mu = (0, \dots, 0,
\underbrace{1}_{\mu\text{-th}}, 0, \dots, 0)
\in \mathbb{R}^n.
\end{align*}
We also define the \emph{interaction kernel} as:
\begin{align*}
S_{BC}(k) = S_C^\top(k)\,S_B(k) \in \mathbb{R}^{n \times n}.    
\end{align*}

\begin{remark}[Dependence on the state]\label{remark:dependence}
Recall that the matrix $D_{ij}(k)$ and its eigenvalues $\lambda_{ij}^\mu(k)$, $1\leq\mu\leq n$, depend on the state $Z(k)=(z_1(k),\hdots,z_\ell(k))$. We will make this dependence explicit when necessary: $D_{ij}(k)=D_{ij}(k,Z(k))$ and $\lambda_{ij}^\mu(k)=\lambda_{ij}^\mu(k,Z(k))$.
\end{remark}

\subsubsection*{Input-output equation}

With these definitions in hand, we eliminate the hidden
state from~\eqref{SSM:recurrence_hidden} and~\eqref{SSM:recurrence_update}, and express the token
update in a closed-form function of the tokens alone.

\begin{lemma}[Input-output formulation of Mamba-2 model]\label{SSM:lemma}
A sequence $Z(k)=(z_1(k),\hdots,z_\ell(k))\in(\mathbb S^{n-1})^\ell$, $k\in\mathbb N$, satisfies \eqref{SSM:recurrence_hidden} and \eqref{SSM:recurrence_update}, for some hidden states, if and only if it satisfies the following recurrence equations:
\begin{align*}
\boxed{z_i(k+1)=\pi\bigg(z_i(k)+\tau\sum_{j=1}^i\beta_{ij}(k)\,z_j(k)\bigg),} 
\end{align*}
for each $i\in\{1,\hdots,\ell\}$ and $k\in\mathbb N$, where:
\begin{align*}
& \beta_{ij}(k)=\left(z_i^\top(k)\,S_{BC}(k)\,z_j(k)\right)D_{ij}(k).
\end{align*}
\end{lemma}

\begin{proof}
For each $i,j\in\{1,\hdots,\ell\}$ with $j\leq i$, a straightforward computation yields:
\begin{align*}
\prod_{l=j+1}^i A_l^\mu & =\prod_{l=j+1}^i\exp(-\alpha^\mu\,\Delta_l^\mu\,\mathbb I_m)\\
& =\prod_{l=j+1}^i\exp(-\alpha^\mu\,\Delta_l^\mu)\,\mathbb I_m\\
& =\exp\left(-\alpha^\mu\sum_{l=j+1}^i\Delta_l^\mu\right)\mathbb I_m=\exp(d_{ij}^\mu)\,\mathbb I_m,
\end{align*}
where we omitted the layer index $k\in\mathbb N$ for brevity.
Hence, since $h_0^\mu=0$, by unrolling \eqref{SSM:recurrence_hidden} we obtain:
\begin{align}\label{proof1}
h_i^\mu & =\sum_{j=1}^{i-1}\left(\prod_{l=j+1}^i A_l^\mu\right)B_j^\mu\,z_j^\mu+B_i^\mu\,z_i^\mu=\sum_{j=1}^i\exp(d_{ij}^\mu)B_j^\mu\,z_j^\mu,
\end{align}
where we used that $\exp(d_{ij}^\mu)=1$ when $i=j$. On the other hand, for each $1\leq\mu\leq n$, we have:
\begin{align}\nonumber
C_i\,\exp(d_{ij}^\mu)\,B_j^\mu\,z_j^\mu & =z_i^\top\,S_C^\top\,\exp(d_{ij}^\mu)\,\Delta_j^\mu\,S_B\,z_j\,z_j^\mu\\\nonumber
& =\left(z_i^\top\,S_C^\top\,S_B\,z_j\right)\Delta_j^\mu\,\exp(d_{ij}^\mu)\,z_j^\mu\,\\
\label{proof2}
& =\left(z_i^\top\,S_{BC}\,z_j\right)\Delta_j^\mu\,\exp(d_{ij}^\mu)\,z_j^\mu.
\end{align}

By substituting \eqref{proof1} into \eqref{SSM:recurrence_update} and taking \eqref{proof2} into account, we obtain:
\begin{align*}
C_i(k)\,h_i^\mu(k) &  =C_i(k)\sum_{j=1}^i\exp(d_{ij}^\mu(k))B_j^\mu(k)\,z_j^\mu(k)\\&=\sum_{j=1}^iz_i^\top(k)\,S_{BC}(k)\,z_j(k)\,\Delta_j^\mu\,\exp(d_{ij}^\mu(k))\,z_j^\mu.
\end{align*}
The result now follows from \eqref{SSM:recurrence_update} as well as the definition of $D_{ij}(k)$.
\end{proof}

\subsection{Continuous-time state space model}\label{sec:contmodel}

The next objective is to derive a continuous-time counterpart of the discrete
model in Lemma~\ref{SSM:lemma}, which will allow us to analyze the dynamics of the Mamba-2 model
with classical control-theoretic tools. To this end, we recall how
continuous-time models are obtained from discrete-time ones on manifolds.

For a compact and connected Riemannian manifold $(M,g)$, let
$\mathbf{d}_g:M\times M\to\mathbb{R}_0^+$ denote the induced geodesic distance. The
flow of a vector field $f\in\mathfrak{X}(M)$ is denoted by
$f^\tau:M\to M$, $\tau\in\mathbb{R}_0^+$. Recall that all vector fields on
a compact manifold are (forward) complete. A map
$\phi:M\times \mathbb{R}\to M$ is a first order approximation of $f^\tau(Z)$
if there exist $T\in \mathbb{R}^+$ and
$\sigma:M\to \mathbb{R}_0^+$ such that
$\mathbf{d}_g(f^\tau(Z),\phi(Z,\tau))\leq\sigma(Z)\,\tau^2$ for each
$\tau\in[0,T]$ and $Z\in M$.

Let $M=(\mathbb S^{n-1})^\ell$ be equipped with the Riemannian metric
induced by the Euclidean metric on $(\mathbb{R}^{n})^{\ell}$. Our objective
is to construct a vector field $f$ on $(\mathbb S^{n-1})^\ell$ whose flow
best approximates the discrete recurrence in Lemma~\ref{SSM:lemma}. Given the
discrete update on Lemma~\ref{SSM:lemma}, the best first-order
approximation in $\tau$ yields:
\begin{align*}
\dot z_i & =\left.\frac{d}{d\tau}\right|_{\tau=0}\pi\left(z_i+\tau\,
\sum_{j=1}^i\beta_{ij}\,z_j\right)=T_{z_i}\pi\cdot\sum_{j=1}^i\beta_{ij}\,z_j.
\end{align*}
By replacing the layer index $k\in\mathbb N$ by a continuous variable
$t\in\mathbb{R}_0^+$, which may be regarded as time, the continuous-time SSM dynamics reads:
\begin{align}\label{SSM:cont}
\boxed{\dot z_i=T_{z_i}\pi\cdot\sum_{j=1}^i\beta_{ij}(t)\,z_j,
\quad1\leq i\leq\ell,}
\end{align}
for $Z=(z_1,\hdots,z_\ell)\in(\mathbb S^{n-1})^\ell$, with:
\begin{align*}
\beta_{ij}(t) & =\left(z_i^\top\,S_{BC}(t)\,z_j\right)D_{ij}(t).
\end{align*}

The weight matrices, together with all quantities defined from them, are now functions
of time:
\begin{align*}
& \alpha^\mu:\mathbb R_0^+\to\mathbb R,
&& S_B,~S_C:\mathbb R_0^+\to\mathbb R^{m\times n},\\
& W_\Delta^\mu:\mathbb R_0^+\to\mathbb R^{1\times n},
&& b_\Delta^\mu:\mathbb R_0^+\to\mathbb R.
\end{align*}

Note that, by compactness of $\mathbb{S}^{n-1}$, the continuous SSM
dynamics is forward complete, \textit{i.e.}, solutions are defined for all
positive times under mild regularity assumptions on the parameter functions.

\begin{remark}[Dependence on the state]
Note that, as in Remark~\ref{remark:dependence}, for each $i,j\in\{1,\hdots,\ell\}$ with $j\leq i$ and $t\in\mathbb R_0^+$, the matrix $D_{ij}(t)$ and its eigenvalues $\lambda_{ij}^\mu(t)$, $1\leq\mu\leq n$, depend on the state $Z(t)=(z_1(t),\hdots,z_\ell(t))$. More generally, we can regard them as depending on two independent inputs: time $t\in\mathbb R_0^+$ and state $Z\in(\mathbb S^{n-1})^\ell$.
We will make this dependence explicit when necessary: $D_{ij}(t,Z)$ and $\lambda_{ij}^\mu(t,Z)$.
\end{remark}

\begin{remark}[Causal structure of SSMs]\label{rem:transformer_cascade}
The continuous-time dynamics~\eqref{SSM:cont} for SSMs is remarkably close
to the continuous-time models for the attention mechanism in
transformers~\cite{rodriguez2025consensus}. Crucially, the dynamics of each
token $z_i$ in~\eqref{SSM:cont} depends only on the preceding tokens
$z_1, \ldots, z_{i-1}$ and itself, endowing the system with a triangular,
or cascade, structure. As in our previous analysis of auto-regressive
transformers~\cite{rodriguez2025consensus}, this cascade structure enables a
systematic study of the asymptotic behavior of the entire model by
leveraging tools from control theory, specifically input-to-state stability (ISS).
\end{remark}

\section{Assumptions on the parameters}

Here we introduce the working assumptions on the parameters of the model. Roughly speaking, they are boundedness of the parameters, uniqueness of the principal eigenvalue of $D_{ij}$ and strict positivity of the spectral gap.

The \emph{spectral gap} is the smallest distance from the principal eigenvalue of $D_{ij}$ to the rest of them:
\begin{align*}
\gamma_{ij}=\inf_{(t,Z)\in\mathbb R_0^+\times(\mathbb S^{n-1})^\ell}~\min_{\mu\neq\tilde\mu}\left(\lambda_{ij}^{\tilde\mu}(t,Z)-\lambda_{ij}^\mu(t,Z)\right),
\end{align*}
for each $i,j\in\{1,\hdots,\ell\}$ with $j\leq i$, where $\lambda_{ij}^{\tilde\mu}(t,Z)$ is the largest eigenvalue of $D_{ij}(t,Z)$ for each $(t,Z)\in\mathbb R_0^+\times(\mathbb S^{n-1})^\ell$. Note that $\gamma_{ij}\in\mathbb R_0^+$. In general, $\tilde\mu$ may depend on the time and state, $(t,Z)$. Our assumption below says that, in fact, it is independent of them. 

\begin{assumption}\label{ass:SSM}
The weight matrices of the continuous model~\eqref{SSM:cont} satisfy:
\begin{enumerate}
    \item There exists $\tilde\mu\in\{1,\hdots,n\}$ such that, for each $(t,Z)\in\mathbb R_0^+\times(\mathbb S^{n-1})^\ell$ and $i,j\in\{1,\hdots,\ell\}$ with $j\leq i$, the principal eigenvalue of $D_{ij}(t,Z)$ is $\lambda_{ij}^{\tilde\mu}(t,Z)$. 
    \item $\gamma=\min_{1\leq j\leq i\leq\ell}~\gamma_{ij}\in\mathbb R^+$.
    \item $\max_{1\leq\mu\leq n}~\sup_{t\in\mathbb R_0^+}|\alpha^\mu(t)|\in\mathbb R^+$.
    \item $\max_{1\leq\mu\leq n}~\sup_{t\in\mathbb R_0^+}\|W_\Delta^\mu(t)\|\in\mathbb R^+$.
    \item $\max_{1\leq\mu\leq n}~\sup_{t\in\mathbb R_0^+}\|b_\Delta^\mu(t)\|\in\mathbb R^+$.
    \item $\min_{1\leq\mu\leq n}~\inf_{t\in\mathbb R_0^+}\|b_\Delta^\mu(t)\|\in\mathbb R^+$.
    \item $\sup_{t\in\mathbb R_0^+}\|S_{BC}(t)\|\in\mathbb R^+$.
\end{enumerate}
\end{assumption}

\medskip

\begin{remark}\label{rem:assumptions}
Items 1) and 2) of Assumption~\ref{ass:SSM} ensure that there is a unique,
strictly dominant principal eigenvalue of $D_{ij}(t,Z)$ for all
$(t,Z)\in\mathbb R_0^+\times(\mathbb S^{n-1})^\ell$.

Note that 1) does
not imply 2): since $t$ ranges over the non-compact set $\mathbb{R}_0^+$,
the spectral gap $\gamma_{ij}(t,Z)$ could be strictly positive for each
$(t,Z)$ yet have infimum equal to zero. Similarly, 2) does not imply 1), as it does not guarantee that $\tilde\mu$ is independent of $(t,Z)$. Nevertheless, 2) implies that the principal eigenvalue of $D_{ij}(t,Z)$ is simple. 

Intuitively, the dominant eigenvalue acts as the primary driving force for
the consensus dynamics, enabling us to prove in
Sections~\ref{sec:PE} and~\ref{sec:global} that tokens converge to the
principal eigenvector $\mathfrak{e}^{\tilde\mu}$ (up to sign) as
$t\to\infty$. Items 3)\,--\,7) are boundedness conditions on the learned
weights. They hold whenever the training procedure produces weights that
remain in a compact set, which is standard in practice.

If 1) in Assumption~\ref{ass:SSM} is violated, the index $\tilde\mu$ of the
dominant eigenvalue may change over time: two eigenvalues could coincide on
an open interval, after which a different component becomes dominant. In Remark~\ref{remark:ass1}, we discuss how this situation affects our analysis.
\end{remark}

Compactness of the sphere, combined with continuity and monotonicity of the exponential and softplus functions, together with the bounds on the weights $W_\Delta^\mu$, $b_\Delta^\mu$ and $\alpha^\mu$, lead to bounds on the matrices $D_{ij}$.

\begin{lemma}\label{lemma:lambdabound}
Under Assumption~\ref{ass:SSM}, for each $i,j\in\{1,\hdots,\ell\}$ with $j\leq i$ and $1\leq\mu\leq n$, we have:
\begin{align*}
& \sup_{(t,Z)\in\mathbb R_0^+\times(\mathbb S^{n-1})^\ell}~\lambda_{ij}^\mu(t,Z)\in\mathbb R^+,\\&\inf_{(t,Z)\in\mathbb R_0^+\times(\mathbb S^{n-1})^\ell}~\lambda_{ij}^\mu(t,Z)\in\mathbb R^+.
\end{align*}
In particular, $\sup_{(t,Z)\in\mathbb R_0^+\times\mathbb S^{n-1}}\|D_{ij}(t,Z)\|\in\mathbb R^+$.
\end{lemma}

To conclude this section, we introduce three sets that will
be used to study the asymptotic stability of the consensus
equilibria and their domains of attraction.

\begin{definition}\label{def:sets}
For each $\mu \in \{1, \dots, n\}$, we define:
\begin{enumerate}
    \item The \emph{consensus set}:
    \begin{align*}
    \mathcal C_\ell(\mu)=\bigcup_{\Sigma\in\{-1,+1\}^\ell}\left\{(\sigma_1\,\mathfrak{e}^{\mu},\hdots,\sigma_\ell\,\mathfrak{e}^{\mu})\in(\mathbb S^{n-1})^\ell\right\},
    \end{align*}
    where $\Sigma = (\sigma_1, \dots, \sigma_\ell)$.
    \item The \emph{spherical cap}:
    \begin{align*}
    \Omega_c^\sigma(\mu) = \{z \in \mathbb{S}^{n-1}
    \mid \sigma\, z^\top \mathfrak{e}^\mu
    = \sigma\, z^\mu > c\},
    \end{align*}
    where $c \in [0,1[$ and $\sigma \in \{-1,+1\}$.
    \item The \emph{equator}: $\mathcal{E}(\mu) = \{z \in \mathbb{S}^{n-1}\mid z^\top \mathfrak{e}^\mu = 0\}$.
\end{enumerate}
\end{definition}

The consensus set $\mathcal{C}_\ell(\mu)$ consists of $2^\ell$
isolated points in $(\mathbb{S}^{n-1})^\ell$, each corresponding to a sign
pattern $\Sigma \in \{-1,+1\}^\ell$. At any consensus equilibrium, every
token is aligned with the eigenvector $\mathfrak{e}^{\mu}$,
up to sign: tokens need not all point in the same direction, but each must lie at one of the two poles\footnote{Here we use the term pole to refer to the antipodal pair $\pm\mathfrak{e}^{\mu}$, that is, the two points where the axis spanned by $\mathfrak{e}^{\mu}$ meets $\mathbb{S}^{n-1}$.} $\pm\mathfrak{e}^{\mu}$. In
Section~\ref{sec:PE}, we prove that each of these equilibria is locally
exponentially stable when $\mu=\tilde\mu$ corresponds to the principal eigenvector of $D_{ij}$. For
comparison, analogous results for
transformers~\cite{rodriguez2024asymptotic,geshkovski2023mathematical}
typically establish attractivity or asymptotic stability on an open
hemisphere. Hence, in that context consensus is understood as all tokens pointing to the same direction. To distinguish both scenarios, the consensus points in $\mathcal C_\ell(\mu)$ are sometimes called \emph{bipartite consensus}, but we will avoid such nomenclature.

\section{Persistency of Excitation\\ of the interaction kernel}\label{sec:PE}

In this section, we prove local exponential stability of the consensus
equilibria under Persistency of Excitation (PE). The remaining cascade
terms are handled by an ISS argument and a recursive choice of sufficiently
small caps. 

\medskip

Following Assumption~\ref{ass:SSM}, for each $i,j\in\{1,\hdots,\ell\}$ with $j\leq i$, we define:
\begin{align*}
\Gamma_{ij} & =\sup_{(t,Z)\in\mathbb R_0^+\times(\mathbb S^{n-1})^\ell}~\max_{\mu\neq\tilde\mu}\bigl(\lambda_{ij}^{\tilde\mu}(t,Z)-\lambda_{ij}^{\mu}(t,Z)\bigr),
\end{align*}
Note that $\Gamma_{ij}\in\mathbb R^+$ thanks to Lemma~\ref{lemma:lambdabound}. This allows for introducing the following parameters:
\begin{align*}
\Gamma & =\max_{1\leq j\leq i\leq\ell}\Gamma_{ij}\in\mathbb R^+,\\
\Lambda & =\max_{2\leq i\leq\ell}\,\sup_{(t,Z)\in\mathbb R_0^+\times(\mathbb S^{n-1})^\ell}\,\sum_{j=1}^{i-1}\lambda_{ij}^{\tilde\mu}(t,Z)\in\mathbb R^+.
\end{align*}

\medskip

Let $c\in\,]0,1[$. For each $(t,z)\in\mathbb R_0^+\times\Omega_c^\sigma(\tilde\mu)$, where $\sigma\in\{-1,+1\}$, we define:
\begin{align}
\label{def:ac}
a_c(t,z)=\begin{cases}
        \gamma c(1+c)\,(z^\top S_{BC}(t)z),
        & z^\top S_{BC}(t)z\in\mathbb R_0^+,\\[0.4em]
        2(\Gamma+\Lambda)\,(z^\top S_{BC}(t)z),
        & z^\top S_{BC}(t)z\in\mathbb R^-.
        \end{cases}
\end{align}

\medskip
We now require that $a_c$ persistently excites the dynamics: there is some $T\in\mathbb R^+$ such that, on average over any interval of size $T$, $a_c$ is positive. In this section we will prove that this is sufficient to establish local exponential stability of the consensus equilibria: intuitively, persistency of excitation allows the first token to converge, and the cascade structure then propagates this convergence through the remaining tokens.

\begin{assumption}[Persistency of excitation]
\label{ass:selfavg}
There exist $c\in\,]0,1[$ and $T,\rho\in\mathbb R^+$ such that $\int_t^{t+T}\alpha_c(s)\,ds\geq\rho T$ for each $t\in\mathbb R_0^+$, where $\alpha_c(t)=\min_{\sigma\in\{-1,+1\}}~\inf_{z\in\Omega_c^\sigma(\tilde\mu)}a_c(t,z)$.
\end{assumption}

\medskip


\begin{remark}[PE and Local Positive Definiteness]\label{remark:positivedefiniteness}
A sufficient condition for PE is $a_c(t,z)\in\mathbb R^+$ for each
$t\in\mathbb R_0^+$ and $z\in\Omega_c^\sigma(\tilde\mu)$ for some $c\in[0,1[$.
This is equivalent to assuming that
$\mathfrak{e}^{\tilde\mu\top} S_{BC}(t)\,\mathfrak{e}^{\tilde\mu}$ is greater than $0$:
positivity at the pole gives, by continuity, positivity on a cap around it,
and conversely positivity on any cap containing the pole gives positivity
at the pole itself. Since $\mathfrak{e}^{\tilde\mu}$ is a coordinate vector,
this is a condition on a single diagonal entry of the interaction kernel,
$\mathfrak{e}^{\tilde\mu\top} S_{BC}(t)\,\mathfrak{e}^{\tilde\mu} =
\big(S_{BC}(t)\big)_{\tilde\mu\tilde\mu}$. PE relaxes it further:
$\big(S_{BC}(t)\big)_{\tilde\mu\tilde\mu}$ is allowed to be negative on
subintervals of time, provided the positive part dominates, on average, over
windows of length $T$.
\end{remark}

Assumption~\ref{ass:selfavg} requires the interaction term to be sufficiently positive on average, rather than pointwise. Whether this property holds depends on the particular weights of the model and architecture considered. For the Mamba-2 architecture studied in Section~\ref{Mamba}, which has a SiLU nonlinearity, the following result provides some motivation for this assumption. In particular, under independently drawn zero-mean weights, the SiLU-modified interaction term has strictly positive expectation, whereas its bilinear counterpart has zero expectation. Thus, while the next result does not establish Assumption~\ref{ass:selfavg}, it shows that, for this particular architectural choice, the SiLU nonlinearity biases the interaction term toward positive values.

\begin{lemma}[Positive bias induced by the nonlinearity]\label{lemma:silu-bias}
Let $S_B, S_C$ be independent random matrices with $\mathbb{E}[S_B] = \mathbb{E}[S_C] = 0$, and let $z_i, z_j \in \mathbb{S}^{n-1}$ be such that, for some $\mu \in \{1,\dots,m\}$, it holds that:
\begin{equation*}
    \begin{aligned}
        \mathbb{P}\big[(S_C z_i)^\mu = 0\big] < 1
  \quad\text{and}\quad
  \mathbb{P}\big[(S_B z_j)^\mu = 0\big] < 1 .
    \end{aligned}
\end{equation*}
Then:
\begin{enumerate}
  \item $\mathbb{E}\big[z_i^\top S_C^\top S_B\, z_j\big] = 0$.\vspace{2mm}
  \item $\mathbb{E}\big[\varsigma(S_C z_i)^\top \varsigma(S_B z_j)\big] > 0$.
\end{enumerate}
\end{lemma}

\begin{proof}
Result 1) follows from independence, since $\mathbb{E}[S_C^\top S_B] = \mathbb{E}[S_C]^\top \mathbb{E}[S_B] = 0$. For 2), the vectors $S_C z_i$ and $S_B z_j$ are independent, and therefore it follows:
\[
  \mathbb{E}\big[\varsigma(S_C z_i)^\top \varsigma(S_B z_j)\big]
  = \mathbb{E}\big[\varsigma(S_C z_i)\big]^\top
    \mathbb{E}\big[\varsigma(S_B z_j)\big].
\]
Since it holds that $(1+e^{-a})^{-1} = \tfrac12 + \tfrac12\tanh(\tfrac{a}{2})$ for each \mbox{$a\in\mathbb{R}$}, the SiLU decomposes as $\varsigma(x) = \tfrac{x}{2} + \varrho(x)$, where \mbox{$\varrho(x)^\mu = \tfrac{x^\mu}{2}\tanh(\tfrac{x^\mu}{2}) \geq 0$}, as $\tanh$ is odd. Computing componentwise, we obtain:
\begin{align*}
    \mathbb{E}[\varsigma(S_C z_i)^\mu] &= \tfrac12 \mathbb{E}[(S_C z_i)^\mu] + \mathbb{E}[\varrho(S_C z_i)^\mu] \\
    &= \mathbb{E}[\varrho(S_C z_i)^\mu] \geq 0,
\end{align*}
for each $\mu \in \{1,\dots,m\}$, since $\mathbb{E}[S_C] = 0$. The same holds for $S_B$, so the inner product above is a sum of products of nonnegative terms. Moreover, $\varrho(a) = 0$ if and only if $a = 0$, so $\mathbb{P}\big[(S_C z_i)^\mu = 0\big] < 1$ yields $\mathbb{E}[\varrho(S_C z_i)^\mu] > 0$, and likewise for $S_B$, so the result follows.
\end{proof}

In Section~\ref{Mamba}, we show the consequence of the positive bias: with weights drawn at random, the tokens cluster when the nonlinearity is present, and spread out when it is removed.

\medskip

Our next step is to introduce the inner product between each token and the candidate equilibrium, and derive its dynamics. This reduces the analysis to a scalar differential equation for each token, in which convergence to consensus amounts to that product reaching one. For convenience, given $\Sigma=(\sigma_1,\hdots,\sigma_\ell)\in\{-1,+1\}^\ell$, tokens are projected to the real line by defining:
\begin{align}\label{PE:bis}
b_i=\sigma_i\,(\mathfrak{e}^{\tilde\mu})^\top\,z_i=\sigma_i\,z_i^{\tilde\mu},\qquad i\in\{1,\hdots,\ell\}.    
\end{align}
The dynamics is obtained from~\eqref{SSM:cont} as $\dot b_i=\sigma_i\,(\mathfrak{e}^{\tilde\mu})^\top\,\dot z_i=\sigma_i\,\dot z_i^{\tilde\mu}$. Note that $z_i\in\Omega_c^{\sigma_i}(\tilde\mu)$ corresponds to $b_i\in\,]c,1]$, and $z_i=\sigma_i\mathfrak e^{\tilde\mu}$ corresponds to $b_i=1$. For convenience, we will also use the following maps as Lyapunov functions:
\begin{align}\label{PE:Vis}
V_i=1-b_i=1-\sigma_i\,\mathfrak e^{\tilde\mu},\quad i\in\{1,\hdots,\ell\}.
\end{align}
Lastly, we define the errors as:
\begin{align}\label{PE:eis}
e_i=z_i-\sigma_i\,\mathfrak{e}^{\tilde\mu},\qquad i\in\{1,\hdots,\ell\}.
\end{align}
An easy check shows that $|e_i|^2 =2\,(1-b_i)$. 

\medskip

The following result, whose proof can be found in the Appendix, is used to show local exponential stability of the consensus equilibria.

\medskip

\begin{proposition}[Bounds for the scalar dynamics]
\label{prop:windowed-ISS}
Let $\Sigma=(\sigma_1,\hdots,\sigma_\ell)\in\{-1,+1\}^\ell$ and consider the projected variables~\eqref{PE:bis} under Assumptions~\ref{ass:SSM} and~\ref{ass:selfavg}. Then:
\begin{enumerate}
    \item For $i=1$, we have $\dot b_1\geq a_c(t,z_1)(1-b_1)$ for each $(t,b_1)\in\mathbb R_0^+\times\,]c,1]$.
    \item For each $i\in\{2,\hdots,\ell\}$ and $\varepsilon\in\mathbb R^+$, there exist $d_\varepsilon\in\,]0,1-c]$ and $M\in\mathbb R^+$ such that $\dot b_i\geq(a_c(t,z_i)-\varepsilon)\,(1-b_i)-M\sum_{j=1}^{i-1}|e_j|$ for each $(t,b_i)\in\mathbb R_0^+\times[1-d_\varepsilon,1]$.
\end{enumerate}
\end{proposition}

\medskip

The first result explaining the asymptotic behavior of \eqref{PE:bis} concerns the first token and it establishes exponential convergence to the principal eigenvector of $D_{ij}$.

\medskip

\begin{lemma}[Asymptotic behavior of the first token]\label{lemma:1stPE}
Under Assumptions~\ref{ass:SSM} and~\ref{ass:selfavg}, for each $\sigma_1\in\{-1,+1\}$, the point $z_{\sigma_1}^*=\sigma_1\,\mathfrak{e}^{\tilde\mu}\in\mathcal C_1(\tilde\mu)$ is a locally exponentially stable equilibrium for the dynamics of the first token in model~\eqref{SSM:cont} and its domain of attraction contains the set $\Omega_{c_1}^{\sigma_1}(\tilde\mu)$, where $c_1=1-(1-c)\exp(\eta_0 T)\in\,]c,1[$ and $\eta_0=\min\left\{0,~\inf_{(t,z)\in\mathbb R_0^+\times\Omega_c^{\sigma_1}(\tilde\mu)}~a_c(t,z)\right\}\in\mathbb R_0^-$.
\end{lemma}
\medskip

\begin{proof}
Note that, using the projected variables~\eqref{PE:bis}, it is enough to show that $b_1^*=1$ is an exponentially stable equilibrium for the dynamics
$\dot b_1=\sigma_1\dot z_1^{\tilde\mu}$ with domain of attraction containing the interval $]c_1,1]$.

Given that $a_c(t,z)\geq\eta_0$ by definition, from the first part of~Proposition~\ref{prop:windowed-ISS}, and using~\eqref{PE:Vis}, we conclude that $\dot V_1\leq-\eta_0\,V_1$ for each $(t,V_1)\in\mathbb R_0^+\times[0,1-c[\,$ and, in particular, for each $(t,V_1)\in\mathbb R_0^+\times[0,1-c_1[\,$. As a result, for each $t\in\mathbb R_0^+$ and $\tau\in[0,T]$ such that $V_1(s)\in[0,1-c_1[$ for each $s\in[t,t+\tau]$, we have:
\begin{align}
\label{eq:first-window-growth}
V_1(t+\tau)\leq\exp(-\eta_0\tau)\,V_1(t).
\end{align}

Next, we show that $V_1(t)\in[0,1-c[$ for each $t\in\mathbb R_0^+$ provided $V_1(0)\in[0,1-c_1[\,$. By contradiction, suppose that $t_0=\inf\{t\in\mathbb R_0^+\mid V_1(t)=1-c\}\in\mathbb R^+$, and write $t_0=(k-1)T+\tau$ for some $k\in\mathbb N$ and $\tau\in[0,T[\,$. Given that $V_1(0)\in[0,1-c_1[\,\subset[0,1-c[\,$, we know that $V_1(t)\in[0,1-c[$ by construction for each $t\in[0,t_0[\,$. Hence, $(1)$ of Proposition~\ref{prop:windowed-ISS} ensures that $\dot V_1(t)\leq -a_c(t,z_1(t))\,V_1(t)$ for each $t\in[0,t_0[\,$. From this and Assumption~\ref{ass:selfavg}, we obtain:
\begin{align*}
V_1((k-1)T) & \leq\exp\left(-\int_{(k-2)T}^{(k-1)T}\alpha_c(t)\,dt\right)V_1((k-2)T)\\
& \leq\exp(-\rho T)\,V_1((k-2)T),
\end{align*}
for each $k\in\mathbb N$. By iterating the previous inequality, we obtain:
\begin{align}\label{proof:1-b0}
V_1((k-1)T) & \leq\exp(-\rho(k-1)T)\,V_1(0),\qquad k\in\mathbb N. 
\end{align}
Thus, $V_1((k-1)T)<V_1(0)<1-c_1$. From this and~\eqref{eq:first-window-growth} with $t=(k-1)T$, we obtain a contradiction:
\begin{align*}
1-c=V_1(t_0) & \leq\exp(-\eta_0\tau)\,V_1((k-1)T)\\
& <\exp(-\eta_0 T)(1-c_1)\\
& =\exp(-\eta_0 T)(1-c)\exp(\eta_0 T)=1-c.
\end{align*}
As a result, $V_1(t)\in[0,1-c[$ for each $t\in\mathbb R_0^+$.

Lastly, for each $t\in\mathbb R_0^+$, let $k\in\mathbb N$ be such that $t\in[(k-1)T,kT[\,$. From~\eqref{eq:first-window-growth} and~\eqref{proof:1-b0}, we conclude:
\begin{align*}
V_1(t) & \leq\exp(-\eta_0(t-(k-1)T))\,V_1((k-1)T)\\
& <\exp(-\eta_0(t-(k-1)T))\exp(-\rho T(k-1))\,V_1(0)\\
& <\exp(-\eta_0T)\exp(-\rho(T(k-1)-t))\exp(-\rho t)\,V_1(0)\\
& <\exp(-\eta_0T)\exp(\rho T)\exp(-\rho t)\,V_1(0)\\
& =M_1\exp(-\rho t)\,V_1(0),
\end{align*}
where $M_1=\exp((\rho-\eta_0)T)\in\mathbb R^+$.
\end{proof}

\medskip

The next result leverages Lemma~\ref{lemma:1stPE} and Proposition~\ref{prop:windowed-ISS} to establish local exponential stability of the consensus set.

\medskip

\begin{theorem}[Local exponential stability of consensus]
\label{theorem:localavg}
Under Assumptions~\ref{ass:SSM} and~\ref{ass:selfavg}, for each
$\Sigma=(\sigma_1,\ldots,\sigma_\ell)\in\{-1,+1\}^\ell$, the consensus
equilibrium
$Z_\Sigma^*=(\sigma_1\,\mathfrak e^{\tilde\mu},\ldots,
\sigma_\ell\mathfrak e^{\tilde\mu})\in\mathcal C_\ell(\tilde\mu)$
is locally exponentially stable for the continuous SSM dynamics~\eqref{SSM:cont}.
\end{theorem}

\medskip

\begin{proof}
Given $\varepsilon\in\,]0,\rho[\,$, let $d_\varepsilon\in\,]0,1-c]$ and $M\in\mathbb R^+$ as in Proposition~\ref{prop:windowed-ISS}. In addition, let:
\begin{align*}
\eta_\varepsilon=\min\left\{\varepsilon,~\min_{\sigma\in\{-1,+1\}}~\inf_{(t,z)\in\mathbb R_0^+\times\Omega_c^\sigma(\tilde\mu)}~a_c(t,z)\right\}-\varepsilon\in\mathbb R_0^-.
\end{align*}
Let $\theta=\exp(-(\rho-\varepsilon)T)\in\,]0,1[\,$, $r=\exp(-\eta_\varepsilon T)\in[1,\infty[$, $R\in\,]r,\infty[$ and $H=T\,M\exp(-\eta_\varepsilon T)\in\mathbb R^+$. Following Lemma~\ref{lemma:1stPE}, let $M_1\in[1,\infty[$ and $\rho_1\in\mathbb R^+$ be such that:
\begin{align}\label{PE:exp1st}
V_1(t)\leq M_1\,\exp(-\rho_1 t)\,V_1(0),
\end{align}
for each $V_1(0)\in[0,d_1[\,$, where $d_1=1-c_1$ and $c_1=1-(1-c)\exp(\eta_0 T)\in\,]c,1[\,$. For $i\in\{2,\hdots,\ell\}$, we recursively define:
\begin{align}\label{eq:rhonui}
& \nu_i=\frac{1}{2}\min\left\{\rho_1,\hdots,\rho_{i-1}\right\},\\\label{eq:rhonui2}
& \rho_i=\frac{1}{2}\min\left\{-\frac{\log(\theta)}{T},\nu_i\right\}\in\mathbb R^+.
\end{align}
In addition, let $C_i=C(T,\theta,\rho_i)\in[1,\infty[$ be as in Lemma~\ref{lemma:discrete-comp} and:
\begin{align}\label{eq:Mi}
M_i\in[\max\{1,\exp(\rho_i T)(C_i r+1-\theta)\},\infty[\,.
\end{align}
Lastly, let $(d_1,\hdots,d_\ell)$ be the sequence given by Lemma~\ref{lemma:cisavg}, and denote $c_i=1-d_i\in\,]c,1[$ for $i\in\{1,\hdots,\ell\}$.

\medskip

For each $i\in\{1,\hdots,\ell\}$, the dynamics of the $i$-th token only depends on the previous ones $j\in\{1,\hdots,i\}$. Thus, we proceed by induction on the token index.

\medskip

\textbf{Base case.} From Lemma~\ref{lemma:1stPE}, the equilibrium $Z_\Sigma^*(1)=\sigma_1\,\mathfrak e^{\tilde\mu}\in\mathcal C_1(\tilde\mu)$ is locally exponentially stable for the subsystem of~\eqref{SSM:cont} given by the first token, and its domain of attraction contains the set $\Omega_{c_1}^{\sigma_1}(\tilde\mu)$. Specifically,~\eqref{PE:exp1st} holds.

\medskip

\textbf{Induction hypothesis.} Given $i\in\{2,\hdots,\ell\}$, the equilibrium $Z_\Sigma^*(i-1)=(\sigma_1\,\mathfrak{e}^{\tilde\mu},\hdots,\sigma_{i-1}\,\mathfrak{e}^{\tilde\mu})\in\mathcal C_{i-1}(\tilde\mu)$, is exponentially stable for the subsystem of~\eqref{SSM:cont} given by the first $i-1$ tokens and its domain of attraction contains the set $\Omega_{c_1}^{\sigma_1}(\tilde\mu)\times\hdots\times\Omega_{c_{i-1}}^{\sigma_{i-1}}(\tilde\mu)$. In terms of~\eqref{PE:Vis}, the previous condition reads:
\begin{align*}
V_j(t)\leq M_j\,d_j\,\exp(-\rho_j\,t),\qquad V_j(0)\in[0,d_j[\,,
\end{align*}
for each $j\in\{1,\hdots,i-1\}$. Analogously, in terms of the error~\eqref{PE:eis}, we may write:
\begin{align}\label{proof:exponential}
|e_j(t)|\leq\sqrt{M_j}\,\exp\left(-\frac{\rho_j}{2}t\right)\,|e_j(0)|,
\end{align}
for each $|e_j(0)|\in[0,\sqrt{2d_j}[\,$.

\medskip

\textbf{Inductive step.}  We need to show that $b_i^*=1$ is a locally exponentially stable equilibrium for the projected dynamics~\eqref{PE:bis} and its domain of attraction contains the set $]c_i,1]$. To that end, we use the following claim, whose proof can be found in the Appendix.

\medskip

\begin{claim}\label{PE:claim} 
Let $k\in\mathbb N\cup\{0\}$ and denote $K=\sum_{j=1}^{i-1}M\sqrt{2M_jd_j}\exp\left(-\frac{\rho_j}{2}(k-1)T\right)$. If $V_i((k-1)T)\in[0,d_i[\,$, then:
\begin{enumerate}
    \item $V_i(t)\in[0,Rd_i[$ for each $t\in[(k-1)T,kT[\,$.
    \item $V_i(kT)\leq\theta\,V_i((k-1)T)+KT\exp(-\eta_\varepsilon T)$.
    \item $V_i(kT)\in[0,d_i[\,$.
\end{enumerate}
\end{claim}

\medskip

An induction argument using Claim~\ref{PE:claim} shows that, if $V_i(0)\in[0,d_i[\,$, then:
\begin{align}\label{claim:1}
V_i(kT)\leq & \theta\,V_i((k-1)T)\\\nonumber
& +H\sum_{j=1}^{i-1}\sqrt{2M_jd_j}\exp\left(-\frac{\rho_j}{2}(k-1)T\right),
\end{align}
for each $k\in\mathbb N\cup\{0\}$.

\medskip

On the other hand, using~\eqref{eq:rhonui},~\eqref{eq:rhonui2} and $(1-\theta)d_i\geq H\sum_{j=1}^{i-1}\sqrt{2M_jd_j}$ (recall Lemma~\ref{lemma:cisavg}), we obtain:
\begin{align*}
& \sum_{j=1}^{i-1}\sqrt{2M_jd_j}\,\exp\left(-\frac{\rho_j}{2}(k-1)T\right)\\
& \leq\sum_{j=1}^{i-1}\sqrt{2M_jd_j}\exp(-\nu_i(k-1)T)\\
& \leq\frac{1-\theta}{H}\exp(-\nu_i(k-1)T)\,d_i,
\end{align*}
for each $k\in\mathbb N$. From this and~\eqref{claim:1}, we obtain:
\begin{align*}
V_i(kT)\leq & \theta\,V_i((k-1)T))+(1-\theta)\,d_i\,\exp(-\nu_i(k-1)T),
\end{align*}
for each $k\in\mathbb N$. Hence, Lemma~\ref{lemma:discrete-comp} with $\nu=\nu_i$, $d=d_i$, $\varrho=\rho_i$ and $y_k=1-b_i(kT)$ for $k\in\mathbb N\cup\{0\}$, leads to:
\begin{align}\label{proof:Cineq}
V_i(kT)\leq C(T,\theta,\rho_i)\,d_i\,\exp(-\rho_i kT),\qquad k\in\mathbb N.
\end{align}

To conclude, let $V_i(0)\in[0,c_i[$ and $t\in\mathbb R_0^+$. For convenience, we write $t=(k-1)T+\tau$ for some $k\in\mathbb N$ and $\tau\in[0,T[\,$. As above, from $1)$ of Claim~\ref{PE:claim} and Proposition~\ref{prop:windowed-ISS}, we obtain:
\begin{align*}
\dot V_i\leq-\eta_\varepsilon\,V_i+M\sum_{j=1}^{i-1}\sqrt{2M_jd_j}\exp\left(-\frac{\rho_j}{2}(k-1)T\right),
\end{align*}
for each $s\in[(k-1)T,t]$, where~\eqref{proof:exponential} has been used. By integrating the previous expression between $(k-1)T$ and $t$, we obtain:
\begin{align*}
V_i(t)\leq & r\,V_i((k-1)T)\\
& +H\,\sum_{j=1}^{i-1}\sqrt{2M_jd_j}\exp\left(-\frac{\rho_j}{2}(k-1)T\right).
\end{align*}

By gathering the previous expression and~\eqref{proof:Cineq}, we conclude:
\begin{align*}
V_i(t) & \leq\big(C(T,\theta,\rho_i)\,r+1-\theta\big)\,d_i\,\exp(-\rho_i(k-1)T)\\
& \leq M_i\,d_i\,\exp(-\rho_i t),
\end{align*}
where we used~\eqref{eq:rhonui},~\eqref{eq:rhonui2} and $H\sum_{j=1}^{i-1}\sqrt{2M_jd_j}\leq(1-\theta)d_i$, as well as~\eqref{eq:Mi}.
\end{proof}

\medskip

\begin{remark}\label{remark:ass1}
Note that 1) in Assumption~\ref{ass:SSM} is key in our proof. When this assumption is violated, our analysis applies on each interval where a single eigenvalue of $D_{ij}$ dominates, in which case the tokens realign to the new principal eigenvector. On the intervals where the principal eigenvalue is not simple, tokens converge to the intersection of the principal eigenspace with the sphere, which defines a submanifold instead of two isolated points.
\end{remark}

\medskip

As mentioned in Remark~\ref{remark:positivedefiniteness}, when $S_{BC}(t)$ is locally positive definite, Assumption~\ref{ass:selfavg} is automatically satisfied, which yields the following particular case of the previous result. 

\medskip

\begin{corollary}
Suppose that Assumption~\ref{ass:SSM} holds and there exists $c\in\,]0,1[$ such that:
\begin{align*}
\inf_{(t,z)\in\mathbb R_0^+\times\Omega_c^{+1}(\tilde\mu)}~z^\top\,S_{BC}(t)\,z\in\mathbb R^+.    
\end{align*}
Then, for each $\Sigma=(\sigma_1,\ldots,\sigma_\ell)\in\{-1,+1\}^\ell$, the consensus
equilibrium
$Z_\Sigma^*=(\sigma_1\,\mathfrak e^{\tilde\mu},\ldots,
\sigma_\ell\mathfrak e^{\tilde\mu})\in\mathcal C_\ell(\tilde\mu)$
is locally exponentially stable for the continuous SSM dynamics~\eqref{SSM:cont}.
\end{corollary}

{
As mentioned in Section~\ref{sec:model}, different Mamba-2 variants employ different nonlinearities, and the model adopted in this paper does not include the SiLU nonlinearity present in the Mamba-2 model considered in Section~\ref{Mamba}. In that architecture, SiLU is applied to the branches producing $B$ and $C$, so that $S_B(k)z_j(k)$ and $S_C(k)z_i(k)$ are replaced by $\varsigma(S_B(k)z_j(k))$ and $\varsigma(S_C(k)z_i(k))$, respectively. Assumption~\ref{ass:selfavg} and Theorem~\ref{theorem:localavg} can be suitably reformulated to account for this modification, and the corresponding arguments carry over \emph{mutatis mutandis} after replacing the bilinear interaction term by its nonlinear counterpart.

}

\section{Global positive definiteness of the interaction kernel}\label{sec:global}

In this section, the objective is to describe the domain of attraction of the consensus equilibria. An attentive reader may note that Assumption~\ref{ass:selfavg} is in practice difficult to verify, as it requires computing an integral of an infimum over the embedding space in which the tokens lie. To obtain a more practical characterization, we strengthen the assumption to global positive definiteness. A PE variant could be considered as well, with qualitatively similar results, nonetheless we adopt the pointwise condition here for cleaner exposition and ease of comparison with the standard positive-definite setting.

\medskip

\begin{assumption}\label{ass:SBCglobal}
There exists $\alpha\in\mathbb R^+$ such that $S_{BC}(t)\succ\alpha\,\mathbb I_n$ for each $t\in\mathbb R_0^+$.
\end{assumption}
\medskip

Let us introduce some parameters that will be helpful to characterize the domain of attraction of the consensus equilibria:

\medskip

\begin{enumerate}
    \item\emph{Anisotropy measure}:
    \begin{align*}
    c_\star=\frac{\sup_{t\in\mathbb R_0^+}\|S_{BC}(t)-((\mathfrak{e}^{\tilde\mu})^\top\,S_{BC}(t)\,\mathfrak{e}^{\tilde\mu})\mathbb I_n\|}{\inf_{(t,z)\in\mathbb R_0^+\times\mathbb S^{n-1}}z^\top\,S_{BC}(t)\,z}.
    \end{align*}
    \item\emph{Uniform eigenvalue upper bound}:
    \begin{align*}
    \lambda_{\max}=\max_{1\leq j\leq i\leq\ell}~\sup_{(t,Z)\in\mathbb R_0^+\times(\mathbb S^{n-1})^\ell}~\lambda_{ij}^{\tilde\mu}(t,Z).
    \end{align*}
    \item\emph{Uniform eigenvalue lower bound}:
    \begin{align*}
    \lambda_{\min}=\min_{1\leq j\leq i\leq \ell}~\inf_{(t,Z)\in\mathbb R_0^+\times(\mathbb S^{n-1})^\ell}~\lambda_{ij}^{\tilde\mu}(t,Z).
    \end{align*}
    \item\emph{Effective threshold}:
    \begin{align*}
    c_\textrm{eff}=\frac{c_\star\,\lambda_{\max}}{\sqrt{(\gamma+(\ell-1)\,\lambda_{\min})^2+c_\star^2\,\lambda_{\max}^2\,(\ell-1)^2}}
    \end{align*}
\end{enumerate}

\medskip

From Lemma~\ref{lemma:lambdabound}, it is clear that $c_\star\in\mathbb R_0^+$, $\lambda_{\max},\lambda_{\min}\in\mathbb R^+$ and $c_\textrm{eff}\in\,[0,1[\,$.
\medskip

\begin{remark}
The anisotropy measure $c_\star$ quantifies how far $S_{BC}(t)$ deviates
from being a scalar multiple of the identity in the
$\mathfrak{e}^{\tilde\mu}$ direction: when $S_{BC}(t)$ is isotropic,
$c_\star = 0$. The effective threshold $c_\textrm{eff}$ determines the
boundary of the domain of attraction in Theorem~\ref{theorem:global}. Note
that $c_\textrm{eff}$ is an increasing function of $c_\star$: the more
anisotropic the interaction kernel, the larger the set of initial conditions
that may fail to reach consensus. In the isotropic case $c_\star = 0$, we
obtain $c_\textrm{eff} = 0$.
\end{remark}

\medskip

Let us introduce the \emph{trapping set} as:
\begin{align*}
\mathcal B_\ell(\tilde\mu)= & \{Z_0\in(\mathbb S^{n-1})^\ell\mid~\exists\,i\in\{1,\hdots,\ell\}\\
& ~\text{such that}~\limsup_{t\to\infty}|z_i^{\tilde\mu}(t,Z_0)|\leq c_\textrm{eff}\},
\end{align*}
where $Z(t,Z_0)=(z_1(\cdot,Z_0),\hdots,z_\ell(\cdot,Z_0)):\mathbb R_0^+\to(\mathbb S^{n-1})^\ell$ denotes the solution of~\eqref{SSM:cont} with initial condition $Z_0$ at $t=0$.

As in section~\ref{sec:PE}, we work with the scalar variables~\eqref{PE:bis}, the functions~\eqref{PE:Vis} and the errors~\eqref{PE:eis}. Analogous to Proposition~\ref{prop:windowed-ISS}, the following result, whose proof can be found in the Appendix, provides bounds that will be useful in the proof of Theorem~\ref{theorem:global}.

\medskip

\begin{proposition}\label{prop:dotbibound}
Let $Z_0\in(\mathbb S^{n-1})^\ell-\mathcal B_\ell(\tilde\mu)$, $\Sigma\in\{-1,+1\}^\ell$, and $c\in\,]c_\textrm{eff},b^\textrm{max}[\,$, where $b^\textrm{max}=\min_{i\in\{1,\hdots,\ell\}}\limsup_{t\to\infty}|b_i(t)|$. Then, under Assumptions~\ref{ass:SSM} and~\ref{ass:SBCglobal}, there exist $\kappa,M\in\mathbb R^+$ such that, for each $t\in I_{c,i}^\Sigma(Z_0,\tilde\mu)=\{t\in\mathbb R_0^+\mid b_i(t)\in[c,1]\}$, we have:
\begin{align*}
\dot b_i & \geq\kappa\,(1-b_i)-M\sum_{j=1}^{i-1}|e_j(t)|,\qquad i\in\{1,\hdots,\ell\}.
\end{align*}
\end{proposition}

\medskip

As in section~\ref{sec:PE}, we start by studying the dynamics of the first token.

\medskip

\begin{lemma}[Asymptotic behavior of first token]\label{lemma:1stglobal}
Under Assumptions~\ref{ass:SSM} and~\ref{ass:SBCglobal}, for each $\sigma_1\in\{-1,+1\}$, the point $z_{\sigma_1}^*=\sigma_1\,\mathfrak{e}^{\tilde\mu}\in\mathcal C_1(\tilde\mu)$ is an exponentially stable equilibrium for the dynamics of the first token of model~\eqref{SSM:cont} with domain of attraction $\Omega_0^{\sigma_1}(\tilde\mu)$.
\end{lemma}

\medskip

\begin{proof}
In terms of the scalar variable~\eqref{PE:bis}, the dynamics of the first token reads:
\begin{align}\label{proof:dotb1}
\dot b_1=(z_1^\top\,S_{BC}(t)\,z_1)\left(\lambda_{11}^{\tilde\mu}(t,z_1)-z_1^\top\,D_{11}(t)\,z_1\right)b_1,
\end{align}
for each $(t,b_1)\in\mathbb R_0^+\times\,[0,1]$. The only equilibria are $b_1^*=-1$, $b_1^*=0$, and $b_1^*=1$, which correspond to $z_1^*=-z_{\sigma_1}^*$, $z_1^*\in\mathcal E(\tilde\mu)$ and $z_1^*=z_{\sigma_1}^*$, respectively. Given that:
\begin{align}\nonumber
& \lambda_{11}^{\tilde\mu}(t,z_1)-z_1^\top\,D_{11}(t,z_1)\,z_1\\\nonumber
& \quad=\sum_{\mu=1}^n\left(\lambda_{11}^{\tilde\mu}(t,z_1)-\lambda_{11}^\mu(t,z_1)\right)(z_1^\mu)^2\\\nonumber
& \quad=\sum_{\mu\neq\tilde\mu}\left(\lambda_{11}^{\tilde\mu}(t,z_1)-\lambda_{11}^\mu(t,z_1)\right)(z_1^\mu)^2\\\label{proof:gap}
& \quad\geq\gamma(1-b_1^2).
\end{align}
where we used~$|z_1|=1$, we conclude that $\dot b_1\geq\alpha\,\gamma\,(1-b_1^2)\,b_1$, for each $b_1\in[0,1]$. Given that $a_1^*=1$ is an exponentially stable equilibrium of the system $\dot a_1=(1-a_1^2)\,a_1$, $a_1\in[0,1]$, with domain of attraction $]0,1]$, the Gr\"onwall--Bellman inequality ensures that $b_1^\star=1$ is an exponentially stable equilibrium of~\eqref{proof:dotb1} with domain of attraction $]0,1]$. This corresponds to $z_1^\star=z_{\sigma_1}^\star$ being an exponentially stable equilibrium of the dynamics of the first token of~\eqref{SSM:cont} with domain of attraction $\Omega_0^{\sigma_1}(\tilde\mu)$.
\end{proof}

\medskip

We are ready to find the domain of attraction and study the stability of the consensus equilibria for the SSM dynamics.

\medskip

\begin{theorem}[Asymptotic stability of consensus]\label{theorem:global}
Under Assumptions~\ref{ass:SSM} and~\ref{ass:SBCglobal}, the consensus set $\mathcal C_\ell(\tilde\mu)$ is exponentially stable for the continuous SSM dynamics~\eqref{SSM:cont} with domain of attraction:
\begin{align*}
& \mathcal D_\ell(\tilde\mu)=(\mathbb S^{n-1})^\ell-\mathcal B_\ell(\tilde\mu)\\
& =\{Z_0\in(\mathbb S^{n-1})^\ell\mid\limsup_{t\to\infty}|z_i^{\tilde\mu}(t,Z_0)|>{c_\textrm{eff}},~1\leq i\leq\ell\}.
\end{align*}
\end{theorem}

\medskip

\begin{proof}
Given that Assumption~\ref{ass:SBCglobal} is stronger than Assumption~\ref{ass:selfavg}, Theorem~\ref{theorem:localavg} holds, whence $\mathcal C_\ell(\tilde\mu)$ is locally exponentially stable. Thus, it is enough to show that its domain of attraction is $\mathcal D_\ell(\tilde\mu)=(\mathbb S^{n-1})^\ell-\mathcal B_\ell(\tilde\mu)$, \textit{i.e.},
\begin{align*}
\lim_{t\to\infty}Z(t,Z_0)\in\mathcal C_\ell(\tilde\mu)~\Leftrightarrow~Z_0\in\mathcal D_\ell(\tilde\mu).
\end{align*}

To specify the dependence on the initial condition $Z_0\in(\mathbb S^{n-1})^\ell$, as well on the signs $\Sigma=(\sigma_1,\hdots,\sigma_\ell)\in\{-1,+1\}^\ell$, the scalar functions~\eqref{PE:bis} and error functions~\eqref{PE:eis} corresponding to the $i$-th component of the solution $Z(t,Z_0)$ of~\eqref{SSM:cont} are denoted by $b_i^\Sigma(t,Z_0)$ and $e_i^\Sigma(t,Z_0)$, respectively, for each $i\in\{1,\hdots,\ell\}$.

\medskip

\begin{description}
    \item[$(\Rightarrow)$] Let $Z_0\in(\mathbb S^{n-1})^\ell$ be such that $\lim_{t\to\infty}Z(t,Z_0)\in\mathcal C_\ell(\tilde\mu)$, \textit{i.e.}, $\lim_{t\to\infty}|b_i^\Sigma(t,Z_0)|=1$ for each $1\leq i\leq\ell$ and $\Sigma\in\{-1,+1\}^\ell$. By contradiction, suppose that $Z_0\in\mathcal B_\ell(\tilde\mu)$. Then there exists $i\in\{1,\hdots,\ell\}$ such that $\limsup_{t\to\infty}z_i(t,Z_0)\not\in\Omega_{c_\textrm{eff}}^\sigma(\tilde\mu)$, i.e.:
    \begin{align*}
    \limsup_{t\to\infty}|b_i^\Sigma(t,Z_0)| & \leq c_\textrm{eff}<1=\lim_{t\to\infty}|b_i^\Sigma(t,Z_0)|.
    \end{align*}
    As a result, $Z_0\not\in\mathcal B_\ell(\tilde\mu)$, \textit{i.e.}, $Z_0\in\mathcal D_\ell(\tilde\mu)$.

    \medskip
    
    \item[$(\Leftarrow)$] Let us proceed by induction in the number of tokens. For each $1\leq i\leq\ell$, we denote $Z^{(i)}=(z_1,\hdots,z_i)\in(\mathbb S^{n-1})^i$ and $\Sigma^{(i)}=(\sigma_1,\hdots,\sigma_i)\in\{-1,+1\}^i$. Similarly, the solution of the subsystem of~\eqref{SSM:cont} given by the first $i$ tokens with initial condition $Z_0^{(i)}\in(\mathbb S^{n-1})^i$ is denoted by $Z^{(i)}(t,Z_0^{(i)})=(z_1^{(i)}(t,Z_0^{(i)}),\hdots,z_i^{(i)}(t,Z_0^{(i)}))$, for each $t\in\mathbb R_0^+$.

    \medskip
    
    \textbf{Base case.} From Lemma~\ref{lemma:1stglobal}, $\mathcal C_1(\tilde\mu)$ is exponentially stable for the subsystem of~\eqref{SSM:cont} given by the first token, and its domain of attraction is $\Omega_0(\tilde\mu)$.  In particular, $\mathcal B_1(\tilde\mu)=\mathcal E(\tilde\mu)$.

    \medskip
    
    \textbf{Induction hypothesis.} Let $1<i\leq\ell$. For each $Z_0^{(i-1)}\in\mathcal D_{i-1}(\tilde\mu)$, we have:
    \begin{align*}
    \lim_{t\to\infty}Z^{(i-1)}(t,Z_0^{(i-1)})\in\mathcal C_{i-1}(\tilde\mu).
    \end{align*}
    Equivalently, there exists $\Sigma^{(i-1)}=(\sigma_1,\hdots,\sigma_{i-1})\in\{-1,+1\}^{i-1}$ such that:
    \begin{align}\label{proof:ejlimit}
    \lim_{t\to\infty} e_j^{\Sigma^{(j)}}(t,Z_0^{(j)})=0,~1\leq j\leq i-1,
    \end{align}
    where $Z_0^{(j)}=(z_{1,0},\hdots,z_{j,0})\in\mathcal D_j(\tilde\mu)$ and $\Sigma^{(j)}=(\sigma_1,\hdots,\sigma_j)\in\{-1,+1\}^j$.
    
    \medskip
    
    \textbf{Inductive step.} For each $Z_0^{(i)}=(z_{1,0},\hdots,z_{i,0})\in\mathcal D_i(\tilde\mu)$, we can pick $c\in\,]c_\textrm{eff},b_i^{\max}[\,$, where $b_i^{\max}=\limsup_{t\to\infty}|z_i^{\tilde\mu}(t,Z_0)|$. In addition, thanks to the induction hypothesis and the fact that $Z_0^{(i-1)}\in\mathcal D_{i-1}(\tilde\mu)$, there exists $\Sigma^{(i-1)}=(\sigma_1,\hdots,\sigma_{i-1})\in\{-1,+1\}^{i-1}$ such that~\eqref{proof:ejlimit} holds. By taking $\kappa,M\in\mathbb R^+$ as in Proposition~\ref{prop:dotbibound}, there exists $t_0\in\mathbb R_0^+$ such that:
    \begin{align}\label{proof:bounderror}
    M\sum_{j=1}^{i-1}|e_j^{\Sigma^{(j)}}(t,Z_0^j)|<\kappa\,(1-c),\qquad t\in[t_0,\infty[\,.
    \end{align}
    Moreover, given that $b_i^{\max}>c$, there exists $t_1\geq t_0$ such that $|z_i^{\tilde\mu}(t_1,Z_0)|>c$. Let $\sigma_i=\operatorname{sign}(z_i^{\tilde\mu}(t_1,Z_0))$ and set $\Sigma^{(i)}=(\sigma_1,\hdots,\sigma_{i-1},\sigma_i)\in\{-1,+1\}^i$. With the notation of Proposition~\ref{prop:dotbibound}, $t_1\in I_{c,i}^{\Sigma^{(i)}}(Z_0^{(i)},\tilde\mu)$, where $b_i^{\Sigma^{(i)}}(t_1,Z_0^{(i)})=\sigma_i\,z_i^{\tilde\mu}(t_1,Z_0^{(i)})$. Let us show that, in fact, $I_{c,i}^{\Sigma^{(i)}}(Z_0^{(i)},\tilde\mu)=[t_1,\infty[\,$. By contradiction, suppose that:
    \begin{align*}
    t_\star=\inf\{t\geq t_1\mid b_i^{\Sigma^{(i)}}(t,Z_0^{(i)})\leq c\}\in\mathbb R^+.
    \end{align*}
    By continuity, $b_i^{\Sigma^{(i)}}(t_\star,Z_0^{(i)})=c$. Hence, from Proposition~\ref{prop:dotbibound} and~\eqref{proof:bounderror}, we obtain (the dependence on the initial condition $Z_0^{(i)}$ is dropped for simplicity):
    \begin{align*}
    \dot b_i^{\Sigma^{(i)}}(t_\star) & \geq\kappa\,(1-b_i^{\Sigma^{(i)}}(t_\star))-M\sum_{j=1}^{i-1}|e_j^{\Sigma^{(j)}}(t_\star)|\\
    & >\kappa\,(1-c)-\kappa(1-c)=0.
    \end{align*}
    This contradicts that $b_i^{\Sigma^{(i)}}(t)>c$ for each $t\in[t_1,t_\star[\,$.

    As a result, the bound in Proposition~\ref{prop:dotbibound} is valid on $I_{c,i}^{\Sigma^{(i)}}(Z_0^{(i)},\tilde\mu)=[t_1,\infty[\,$. This, together with Lemma~\ref{lemma:comparison} and~\eqref{proof:ejlimit}, gives the result.
    \end{description}
\end{proof}

\medskip

For each token $1\leq i\leq\ell$, the effective threshold $c_\textrm{eff}$ separates two regimes on $\mathbb{S}^{n-1}$. When $|z_i^{\tilde\mu}| > c_\textrm{eff}$, both the self-drift and the isotropic part of the cross-terms drive $b_i(t)$ towards $\pm1$ as $t\to\infty$. When $|z_i^{\tilde\mu}| \le c_\textrm{eff}$, however, the anisotropic cross-terms oppose the self-drift. The dominating term at the current time will determine the sign of $\dot b_i(t)$, which will change within time.

From Lemma~\ref{lemma:1stglobal}, the trapping set of the first token is the equator $\mathcal E(\tilde\mu)$, which has zero measure (as it has co-dimension 1). In other words, the domain of attraction $\mathcal D_1(\tilde\mu)$ is co-null. This follows from the fact that the dynamics of the first token has no cross-terms.

Let us analyze the size of the domain of attraction for the remaining tokens. We can distinguish the two cases:

\medskip

\begin{enumerate}
    \item\textbf{Isotropic case.}  When $S_{BC}(t) = g(t)\,\mathbb I_n$ for some continuous $g:\mathbb R_0^+\to[\alpha,\infty[\,$, with $\alpha\in\mathbb R^+$, the effective threshold vanishes: $c_\textrm{eff} = 0$. In this case, the trapping set is defined by an equality and, thus, it has zero measure (as it has positive co-dimension). Therefore, the domain of attraction $\mathcal D_\ell(\tilde\mu)$ is co-null.
    \item\textbf{Anisotropic case.} When $S_{BC}(t)$ is not a multiple of the identity, $c_\star\in\mathbb R_0^+$ and the trapping set may have positive measure. Furthermore, there might be locally asymptotically stable equilibria inside the trapping set. The size of the domain of attraction is a decreasing function of the effective threshold.
\end{enumerate}

\section{Mamba-2 Experiments}
\label{Mamba}
In this section, our objective is to verify whether the consensus phenomenon can be observed in the Mamba-2 model as the number of layers increases. To do so, we use the \texttt{mamba2-130m} model ($n = 768$, $24$ layers, $64$ channels per head) from~\cite{mamba2-130m}. Since our results are asymptotic, we extend the depth by cycling the $24$ available layers: once the tokens have passed through the last layer, they re-enter the first, so that the setting remains time-varying. We take $\ell = 50$ tokens, initialised by uniformly sampling from the dictionary, and iterate the layer update for a depth of $200$. Each experiment is repeated for $50$ independent draws of the initial tokens, and we report the mean over those runs.

Writing $Z = (z_1,\dots,z_\ell)$, let $\bar\lambda_1 \ge \dots \ge
\bar\lambda_\ell$ denote the eigenvalues of the Gram matrix of the normalised
tokens $z_i/|z_i|$, scaled so that $\sum_k \bar\lambda_k = 1$. We measure two
quantities:

\begin{itemize}
  \item $\bar\lambda_1$, the share of the total energy carried by the leading
        direction. It equals $1$ exactly when all tokens are collinear, and
        $1/\ell$ when they are mutually orthogonal.
  \item $r = \big(\sum_k \bar\lambda_k^2\big)^{-1}$, the participation ratio of
        that spectrum: the effective number of directions occupied by the
        tokens, ranging from $1$ (collinear) to $\ell$ (isotropic).
\end{itemize}

Both  detect
alignment along a common direction irrespective of the signs $\sigma_i$, \textit{i.e}, at any
consensus configuration in $\mathcal{C}_\ell(\tilde\mu)$ they satisfy
$\bar\lambda_1 = 1$ and $r = 1$.

\subsection{Results}

\begin{figure}[t]
  \centering
  \includegraphics[width=\columnwidth]{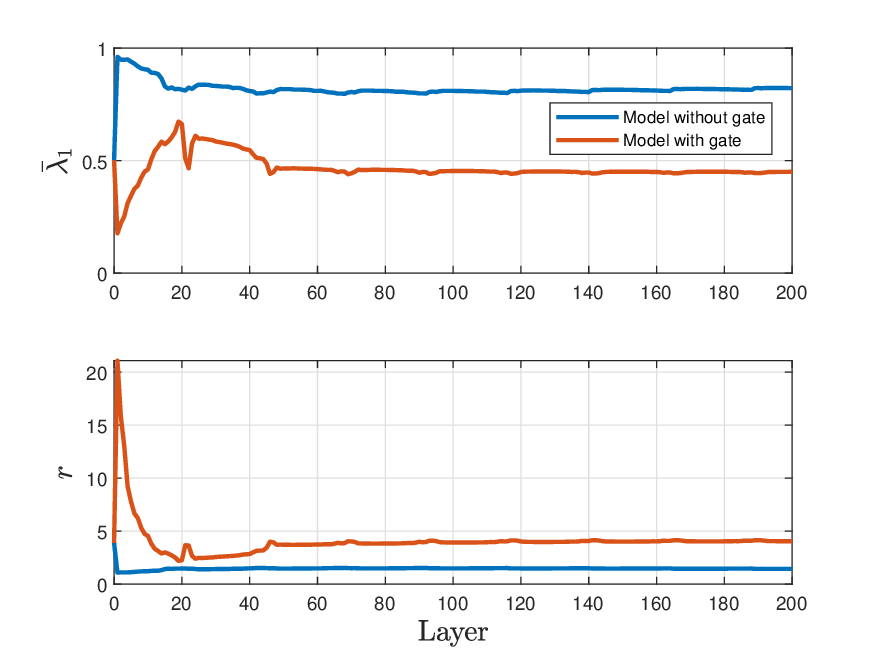}
  \caption{Weights of the \texttt{mamba2-130m} model, with and without the
    output gate. Top: $\bar\lambda_1$. Bottom: $r$.}
  \label{fig:pretrained}
\end{figure}

Figure~\ref{fig:pretrained} reports the values of $r$ and $\bar{\lambda}_1$ obtained using the original Mamba-2 weights and the SiLU nonlinearity, both with and without the output gate.
  Recall from Section~\ref{sec:model} that the output gate is a nonlinear function $g(z_i)$ that multiplies the output of the recurrence elementwise and is excluded from our model. The results with and without the output gate follow the same qualitative trajectory: the tokens aggregate along a single direction over the first layers, after which the two quantities reach a plateau and remain there for the rest of the run. What differs is how far the aggregation proceeds before settling. Without the gate, $\bar\lambda_1$ rises to approximately $0.8$ and $r$ falls to about $1.4$, against roughly $0.45$ and $4$ when the gate is retained. This supports the claim that the gate is what keeps the tokens from converging to a single direction.

\begin{figure}[t]
  \centering
  \includegraphics[width=\columnwidth]{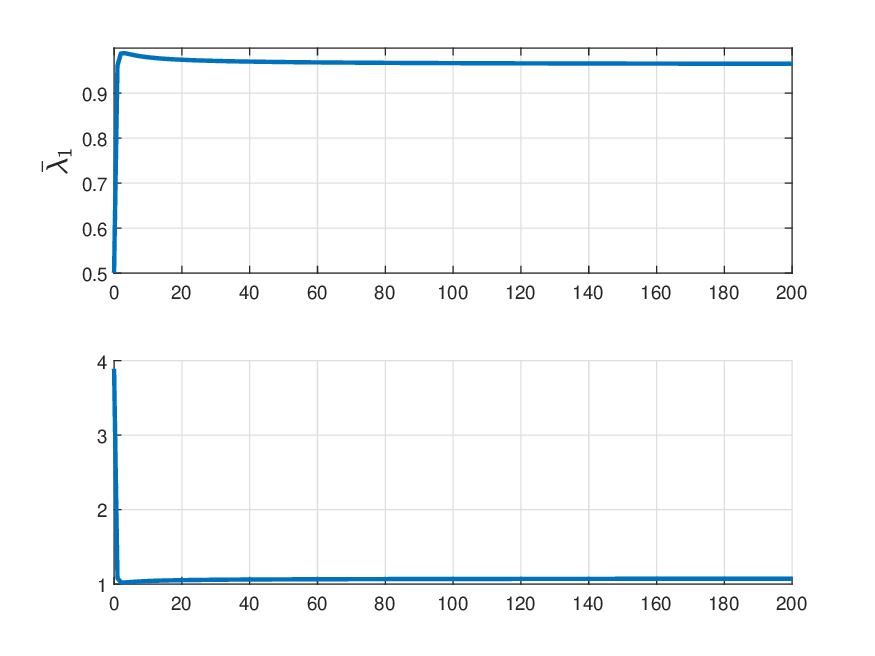}
  \caption{A single layer of the \texttt{mamba2-130m} model applied repeatedly,
    with the output gate removed. Top: $\bar\lambda_1$. Bottom: $r$.}
  \label{fig:single}
\end{figure}

In the second experiment, our objective was to investigate how the consensus phenomenon depends on the time-varying nature of the weights.
 To do so, we repeated the previous experiment without the gate, using the same
layer at every depth, so that the weight matrices are time-invariant and, in
particular, the dominant eigenvalue of $D_{ij}$ is attained at the same index
$\tilde\mu$ throughout. In Figure~\ref{fig:single}, it can be seen that the behaviour changes qualitatively: $\bar\lambda_1$ exceeds $0.95$ within the first few layers and $r$ drops to approximately $1.05$, both remaining there for the rest of the run, so that the tokens are aligned along a single direction up to numerical accuracy.

Comparing the two experiments, the weaker consensus in Figure~\ref{fig:pretrained} can be attributed to the time variation of the weights. When the weights vary with the layer, the index $\tilde\mu$ of the dominant eigenvalue of $D_{ij}$ need not remain constant, and as mentioned in Remark~\ref{ass:SSM}, each time the index changes the tokens begin realigning towards a new principal eigenvector. The convergence established on each interval of constancy is therefore interrupted before the tokens reach the corresponding equilibrium. Holding the weights fixed removes these transitions, and the consensus phenomenon can be observed in full.

\begin{figure}[t]
  \centering
  \includegraphics[width=\columnwidth]{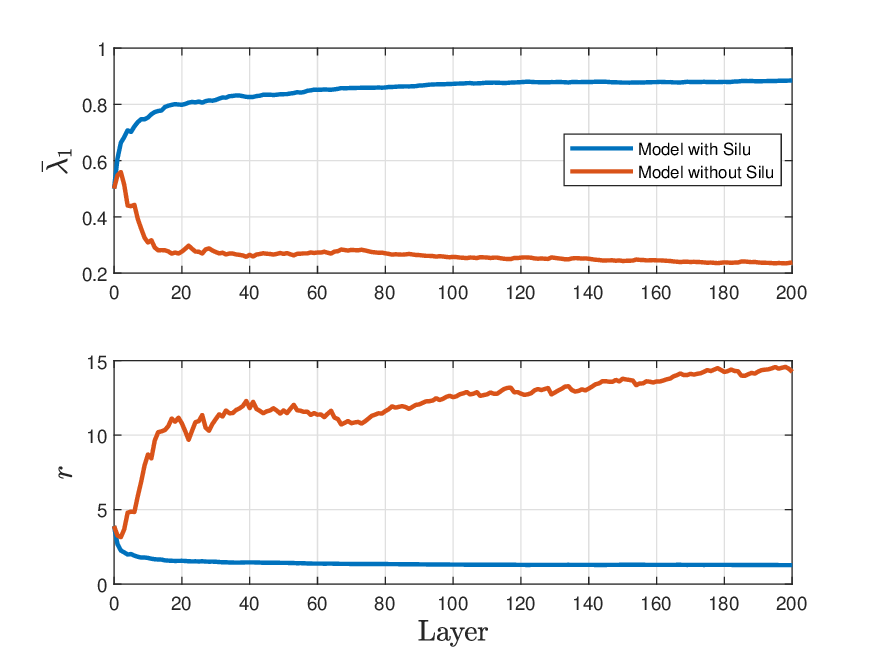}
  \caption{Randomly drawn weight matrices, redrawn at every layer and with the
    output gate removed, with and without the nonlinearity on the branch
    producing $B$ and $C$. Top: $\bar\lambda_1$. Bottom: $r$.}
  \label{fig:nosilu}
\end{figure}

In the third experiment our objective was to isolate the effect of the SiLU, the elementwise nonlinearity \mbox{$\varsigma(a) = a\,(1+e^{-a})^{-1}$}, $a \in \mathbb{R}$, present in the Mamba-2 architecture used in this section. To do so, we draw the weight matrices independently at every layer, so that the layers are not periodic, and remove the gate, running the model with and without the SiLU. Figure~\ref{fig:nosilu} reports the values of $r$ and $\bar \lambda_1 $ for the random model with and without SiLU.

It can be seen that, for both $\bar\lambda_1$ and $r$, the curves with and
without the SiLU separate within the first few layers and then move in opposite
directions. With the nonlinearity, $\bar\lambda_1$ increases steadily to
approximately $0.88$ while $r$ decreases to about $1.3$, so the tokens
concentrate along a single direction. Without it, $\bar\lambda_1$ falls to
roughly $0.23$ while $r$ grows past $14$ and is still increasing at the end of
the run: the tokens do not merely fail to reach consensus, they spread over an
increasing number of directions.

Lemma~\ref{lemma:silu-bias} accounts for the difference between the experiments with and without the SiLU. With matrices drawn independently at every layer, which is the setting of this experiment, the interaction term is centered at $0$ without the nonlinearity, whereas the SiLU shifts its mean to be strictly positive. While this does not establish that Assumption~\ref{ass:selfavg} is satisfied, and hence does not guarantee that the tokens converge, it indicates that the interaction term is biased towards positive values, which is what drives the tokens together.

\subsection{Persistency of Excitation}

Our results rely on Assumption~\ref{ass:selfavg} and in this section we report on experiments designed to test it. We therefore turn to the
assumption itself. For the first token, item~1) of
Proposition~\ref{prop:windowed-ISS} reads $\dot b_1 \ge a_c(t,z_1)(1-b_1)$, so
the quantity that governs its convergence is $a_c$ evaluated along the
trajectory. We therefore record, at every layer and for each of the $50$ runs, the
interaction term of each token:
\[
  q_i(t) \;=\; \varsigma\big(S_C(t)\,z_i(t)\big)^{\!\top} \varsigma\big(S_B(t)\,z_i(t)\big),
\]
which reduces to $z_i(t)^\top S_{BC}(t)\, z_i(t)$ when the nonlinearity is
removed, and which agrees with $a_c(t, z_i(t))$ up to a positive constant, for
each $t$. For a window length $T$ we compute the worst-case average of $q_i$
over all windows of that length contained in the run, and call $T$ admissible
when it is positive. In the first two results we look for the smallest
admissible window:
\begin{equation}\label{eq:T}
  T^\star \;=\; \min\left\{\, 1\leq T\leq \kappa  \;\middle|\;
    \min_{0 \le t \le \kappa - T}\;\frac{1}{T}\sum_{s=t+1}^{t+T} q_i(s) \;>\; 0 \,\right\},
\end{equation}
where the inner minimum is taken over all windows contained in the run, and
$T^\star$ does not exist when no window length is admissible. Note that the existence of $T^*$ does not establish Assumption~\ref{ass:selfavg}, which requires
the averaged condition at every point of the cap rather than along the sampled
trajectories. The two are nonetheless related: whenever $z_i(t)$ lies in the
cap, $q_i(t)$ is an upper bound for $\alpha_c(t)$ up to a positive constant, so
a negative average refutes the assumption, while a positive one across many
trajectories is supporting evidence for it, without establishing it.

Figures~\ref{fig:pe-model} and~\ref{fig:pe-random} illustrate the measurement on
the first token of one run, with the original weights of the Mamba-2 and for randomly drawn
weights respectively, both without the output gate. The sign of $q_1$ alternates in both figures, so although pointwise positivity fails, $q_1(t)$ remains positive on average.
 The minimum average is negative for short windows, but increases
with the window length, and crosses zero at $T = 8$ for the model and $T = 34$
for the random weights. With windows of at least that length the interaction term is positive on average
along these trajectories, and hence so is the upper bound it provides for
$\alpha_c$.

Figure~\ref{fig:pe-summary} collects the four configurations, taking the largest $T^*$ of \eqref{eq:T} over all tokens and all runs. With the original Mamba-2 weights the smallest
admissible window is $T = 10$, and removing the nonlinearity leaves it
essentially unchanged at $T = 11$. For matrices drawn independently at every
layer, the smallest admissible window is $T = 35$ with the nonlinearity, and no
window length up to the length of the run is admissible without it, $i.e.$ along
these trajectories Assumption~\ref{ass:selfavg} is violated. This behaviour is consistent with Lemma~\ref{lemma:silu-bias}. The nonlinearity therefore accounts for the clustering seen with random
matrices, while for the trained matrices the interaction term is positive on
average along these trajectories whether or not the nonlinearity is present.

Overall, along the trajectories we sample, the interaction term of the trained
Mamba-2 is positive on average over short windows compared with the depth
simulated. This does not establish Assumption~\ref{ass:selfavg}, but it is
consistent with it, and indicates that the mechanism analysed in
Section~\ref{sec:PE} is at work in Mamba-2 itself and not only in the setting
under which our results were proved.

\begin{figure}[t]
  \centering
  \includegraphics[width=\columnwidth]{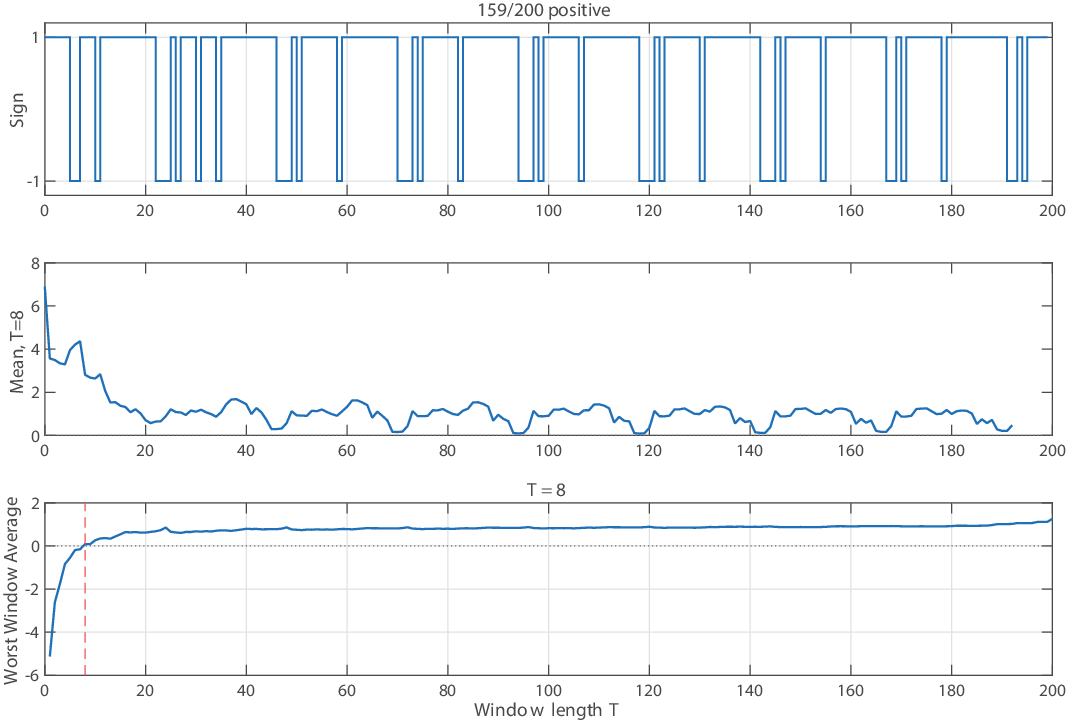}
  \caption{Interaction term of the first token along one trajectory of the
    \texttt{mamba2-130m} model. Top: its sign at each layer. Middle: its average
    over windows of length $T = 8$. Bottom: the worst-case window average as a
    function of the window length, with the smallest admissible $T$ marked.}
  \label{fig:pe-model}
\end{figure}

\begin{figure}[t]
  \centering
  \includegraphics[width=\columnwidth]{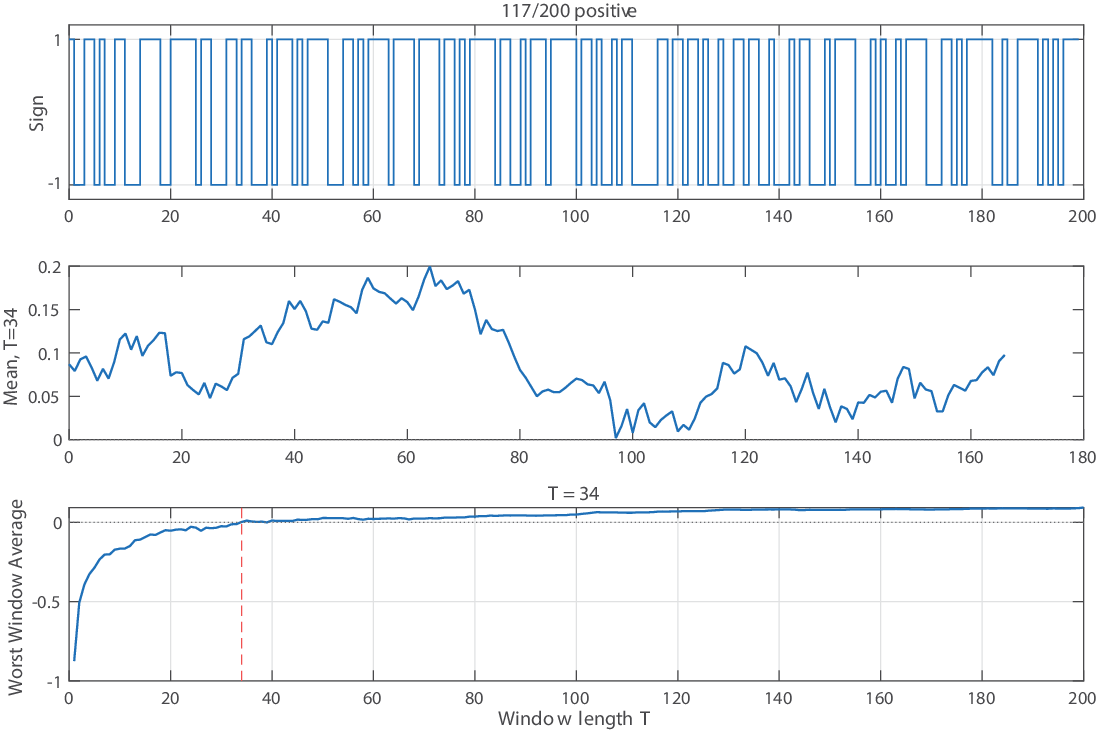}
  \caption{As in Figure~\ref{fig:pe-model}, for weight matrices drawn
    independently at every layer and with the output gate removed. The smallest
    admissible window is longer, $T = 34$.}
  \label{fig:pe-random}
\end{figure}

\begin{figure}[t]
  \centering
  \includegraphics[width=\columnwidth]{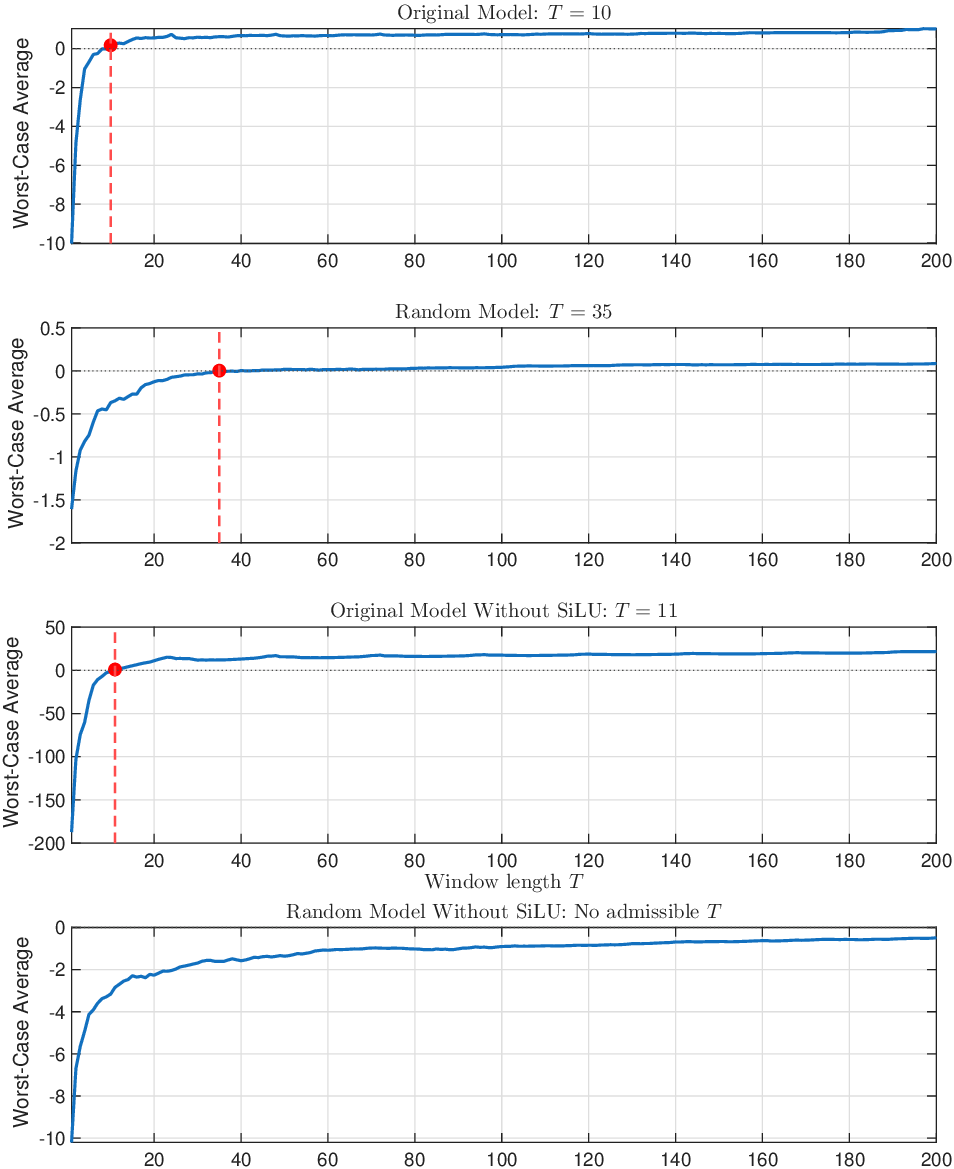}
  \caption{Worst-case window average of the interaction term of the first token
    against the window length, for the four configurations considered: the
    \texttt{mamba2-130m} weights and independently drawn weights, each with and
    without the nonlinearity on the branches producing $B$ and $C$. The smallest
    admissible window is marked where it exists; for independently drawn weights
    without the nonlinearity no window length is admissible.}
  \label{fig:pe-summary}
\end{figure}

\section{Conclusion}

In this paper, we showed that the relationship between selective state space models and transformers extends beyond their input-output representations and reaches the dynamics induced by depth. By deriving a continuous-time model for Mamba-2 with time-varying weight matrices, we recast the evolution of tokens as a dynamical system on the sphere. In this framework, using the causal cascade structure of selective SSMs and input-to-state stability, we proved local exponential stability of the consensus equilibria under a persistency of excitation condition. Then, under a stronger global condition, we described the corresponding domain of attraction. Our experiments further indicate that the output gate is the component that attenuates the consensus phenomenon, which suggests that gating is fundamental in preventing all tokens from converging to a single cluster. These results provide, to the best of our knowledge, the first consensus analysis for selective SSMs with time-varying weight matrices, showing that the mechanisms underlying their efficient recurrent structure also lead to the loss of token diversity.

\bibliographystyle{ieeetr}
\bibliography{biblio.bib}

\appendix


We begin by showing the bounds for the projected variables in section~\ref{sec:PE}.

\medskip

\begin{proof}[of Proposition~\ref{prop:windowed-ISS}]
For each $(t,Z)\in\mathbb R_0^+\times(\mathbb S^{n-1})^\ell$ and $i,j\in\{1,\hdots,\ell\}$ with $j\leq i$, we may write:
\begin{align*}
z_i^\top\,S_{BC}(t)\,(e_j+\sigma_j\mathfrak{e}^{\tilde\mu})=z_i^\top\,S_{BC}(t)\,e_j+\sigma_j\,z_i^\top\,S_{BC}(t)\,\mathfrak{e}^{\tilde\mu}.
\end{align*}
This allows for expressing the dynamics~\eqref{SSM:cont} of the $i$-th token as follows:
\begin{align}\nonumber
\dot z_i=\, & T_{z_i}\pi\cdot(z_i^\top\,S_{BC}(t)\,z_i)D_{ii}(t,Z)\,z_i\\\nonumber
& +T_{z_i}\pi\cdot\sum_{j=1}^{i-1}(z_i^\top S_{BC}(t)\,\mathfrak{e}^{\tilde\mu})D_{ij}(t,Z)\,\mathfrak{e}^{\tilde\mu}\\\nonumber
& +T_{z_i}\pi\cdot\sum_{j=1}^{i-1}\sigma_j\,(z_i^\top S_{BC}(t)\,\mathfrak{e}^{\tilde\mu})D_{ij}(t,Z)\,e_j\\\nonumber
& +T_{z_i}\pi\cdot\sum_{j=1}^{i-1}\sigma_j\,(z_i^\top S_{BC}(t)\,e_j)D_{ij}(t,Z)\,\mathfrak{e}^{\tilde\mu}\\\nonumber
& +T_{z_i}\pi\cdot\sum_{j=1}^{i-1}(z_i^\top S_{BC}(t)\,e_j)D_{ij}(t,Z)\,e_j,
\end{align}
where the sums are zero when $i=1$ and we used that $\sigma_j^2=1$ for each $j\in\{1,\hdots,i-1\}$. Hence (we omit the arguments $t$ and $Z$ for brevity):
\begin{align}\label{PE:dotbi1}
\dot b_i =~ & (z_i^\top\,S_{BC}\,z_i)\,(\lambda_{ii}^{\tilde\mu}-z_i^\top\,D_{ii}\,z_i)\,b_i\\\label{PE:dotbi2}
& +\sum_{j=1}^{i-1}(z_i^\top\,S_{BC}\,\mathfrak{e}^{\tilde\mu})(\sigma_i\lambda_{ij}^{\tilde\mu}-(z_i^\top\,D_{ij}\,\mathfrak e^{\tilde\mu})\,b_i)\\\label{PE:dotbi3}
& +\sum_{j=1}^{i-1}\sigma_j\,(z_i^\top\,S_{BC}\,\mathfrak{e}^{\tilde\mu})\,(\sigma_i\lambda_{ij}^{\tilde\mu}\,e_j^{\tilde\mu}-(z_i^\top\,D_{ij}\,e_j)\,b_i)\\\label{PE:dotbi4}
& +\sum_{j=1}^{i-1}\sigma_j\,(z_i^\top\,S_{BC}\,e_j)(\sigma_i\lambda_{ij}^{\tilde\mu}-(z_i^\top\,D_{ij}\,\mathfrak e^{\tilde\mu})\,b_i)\\\label{PE:dotbi5}
& +\sum_{j=1}^{i-1}(z_i^\top\,S_{BC}\,e_j)(\sigma_i\lambda_{ij}^{\tilde\mu}\,e_j^{\tilde\mu}-(z_i^\top\,D_{ij}\,e_j)\,b_i).
\end{align}

From the fact that $|z_i|=1$, we obtain $\gamma(1-b_i^2)\leq\lambda_{ii}^{\tilde\mu}-z_i^\top D_{ii} z_i\leq\Gamma(1-b_i^2)$. Hence, given $(t,b_i)\in\mathbb R_0^+\times\,]c,1]$, we distinguish two cases:
\begin{enumerate}
    \item $z_i^\top\,S_{BC}\,z_i\in\mathbb R_0^+$, then:
    \begin{align*}
    \eqref{PE:dotbi1} & \geq(z_i^\top\,S_{BC}\,z_i)\gamma(1+b_i)(1-b_i)b_i\\
    & \geq\gamma c(1+c)(z_i^\top\,S_{BC}\,z_i)\,(1-b_i)=a_c(t,z_i)\,(1-b_i).
    \end{align*}
    \item $z_i^\top\,S_{BC}\,z_i\in\mathbb R^-$, then:
    \begin{align*}
    \eqref{PE:dotbi1} & \geq(z_i^\top\,S_{BC}\,z_i)\Gamma(1+b_i)(1-b_i)b_i\\
    & \geq 2(\Gamma+\Lambda)(z_i^\top\,S_{BC}\,z_i)\,(1-b_i)=a_c(t,z_i)\,(1-b_i),
    \end{align*}
    where we used that $(1+b_i)b_i\leq2$ and $\Lambda\in\mathbb R^+$.
\end{enumerate}

For $i=1$,~\eqref{PE:dotbi2},~\eqref{PE:dotbi3},~\eqref{PE:dotbi4} and~\eqref{PE:dotbi5} vanish, so we conclude that $\dot b_1\geq a_c(t,z_1)\,(1-b_1)$.

For $i\in\{2,\hdots,\ell\}$, we need to bound the remaining terms. Using that $\sigma_i\,\mathfrak e^{\tilde\mu}=z_i-e_i$, we obtain:
\begin{align*}
\eqref{PE:dotbi2} & =\sum_{j=1}^{i-1}\sigma_i(z_i^\top\,S_{BC}\,\mathfrak{e}^{\tilde\mu})\,\lambda_{ij}^{\tilde\mu}\,(1-b_i^2)\\
& =\sum_{j=1}^{i-1}(z_i^\top\,S_{BC}\,z_i-z_i^\top\,S_{BC}\,e_i)\,\lambda_{ij}^{\tilde\mu}\,(1-b_i^2)\\
& =\sum_{j=1}^{i-1}\lambda_{ij}^{\tilde\mu}\,(z_i^\top\,S_{BC}\,z_i-z_i^\top\,S_{BC}\,e_i)\,(1-b_i^2).
\end{align*}
As for the first token, given $(t,b_i)\in\mathbb R_0^+\times\,]c,1]$, we distinguish two cases:
\begin{enumerate}
    \item $z_i^\top\,S_{BC}\,z_i\in\mathbb R_0^+$, then $\eqref{PE:dotbi2}\geq-\Lambda\,(z_i^\top\,S_{BC}\,e_i)\,(1-b_i^2)$.
    \item $z_i^\top\,S_{BC}\,z_i\in\mathbb R^-$, then $\eqref{PE:dotbi2}\geq2\Lambda\,(z_i^\top\,S_{BC}\,z_i)\,(1-b_i)-\Lambda\,(z_i^\top\,S_{BC}\,e_i)\,(1-b_i^2)$, where we used that $1+b_i\leq2$ and $\Lambda\in\mathbb R^+$.
\end{enumerate}
Therefore, in both cases we obtain:
\begin{align*}
\eqref{PE:dotbi1}+\eqref{PE:dotbi2}\geq a_c(t,z_i)\,(1-b_i)-\Lambda\,(z_i^\top\,S_{BC}\,e_i)\,(1-b_i^2).
\end{align*}
Let us pick $d_\varepsilon\in\,]0,1-c]$ such that $2\sqrt{2d_\varepsilon}\Lambda\sup_{t\in\mathbb R_0^+}\|S_{BC}(t)\|\leq\varepsilon$. Hence, for each $(t,b_i)\in\mathbb R_0^+\times\,]1-d_\varepsilon,1]$, we have $|e_i|=\sqrt{2(1-b_i)}<\sqrt{2d_\varepsilon}$, whence:
\begin{align*}
-\Lambda\,(z_i^\top\,S_{BC}\,e_i)\,(1-b_i^2) & \geq -2\sqrt{2d_\varepsilon}\Lambda\,\sup_{t\in\mathbb R_0^+}\|S_{BC}(t)\|\,(1-b_i)\\
& \geq-\varepsilon\,(1-b_i),
\end{align*}
where we used that $1+b_i\leq 2$ and $|z_i|=1$. As a result:
\begin{align*}
\eqref{PE:dotbi1}+\eqref{PE:dotbi2}\geq(a_c(t,z_i)-\varepsilon)\,(1-b_i).
\end{align*}

Lastly, from Assumption~\ref{ass:SSM} and Lemma~\ref{lemma:lambdabound}, as well as compactness of the sphere, the error terms are all bounded, \textit{i.e.}, there exists $M\in\mathbb R^+$ such that:
\begin{align}
\label{eq:residual-bound}
        \eqref{PE:dotbi3}+\eqref{PE:dotbi4}+\eqref{PE:dotbi5}
        \geq -M\sum_{j=1}^{i-1}|e_j|.
\end{align}
for each $(t,b_i)\in\mathbb R_0^+\times\,]c,1]$, and in particular, for each $(t,b_i)\in\mathbb R_0^+\times\,]1-d_\varepsilon,1]$.

By gathering the previous bounds, we conclude.
\end{proof}

\medskip

The following two technical lemmas are useful in the proof of Theorem~\ref{theorem:localavg}.

\medskip

\begin{lemma}
\label{lemma:cisavg}
Let $d_\varepsilon,H,M_i,\theta,r,R\in\mathbb R^+$, $i\in\{1,\hdots,\ell\}$, be such that $\theta<1$ and $R>r\geq 1$. Then there exists a sequence:
\begin{align*}
(d_1,\hdots,d_\ell)\in\prod_{i=1}^\ell\,]d_i^{\min},d^{\max}[\,,
\end{align*}
where $d^{\max}=d_\varepsilon/R$, $d_1^{\min}=0$ and:
\begin{align*}
d_i^{\min}=\max\left\{d_{i-1},~\frac{H}{\mu}\sum_{j=1}^{i-1}\sqrt{2M_jd_j}\right\},
\end{align*}
for $i\in\{2,\hdots,\ell\}$, with $\mu=\min\{1-\theta,R-r\}\in\mathbb R^+$.
\end{lemma}

\medskip

\begin{proof}
Given $d_1\in\,]0,d^{\max}[\,$, we recursively define:
\begin{align*}
        d_i
        =
        2d_{i-1}
        +
        \frac{2H}{\mu}
        \sum_{j=1}^{i-1}\sqrt{2M_jd_j},
\end{align*}
for each $i\in\{2,\hdots,\ell\}$. It is clear that $d_i>d_i^{\min}$ by construction.

To conclude, we need to show that $d_1$ can be chosen so that $d_i<d^{\max}$ for each $i\in\{2,\hdots,\ell\}$. To that end, we regard $d_i=d_i(d_1)$, and note that they are continuous functions on $]0,d^{\max}[\,$. Hence, $\lim_{d_1\to0^+}d_i(d_1)=d_i(0)=0$. Hence, $d_1\in\,]0,d^{\max}[$ can be chosen so that $d_i<d^{\max}$ for each $i\in\{2,\hdots\ell\}$.
\end{proof}

\medskip

\begin{lemma}[Discrete ISS]
\label{lemma:discrete-comp}
Let $\theta\in\,]0,1[\,$, $\nu,d\in\mathbb R^+$, and
$\varrho=\frac{1}{2}\min\{-\log(\theta)/T,\nu\}$. There exists $C(T,\theta,\varrho)\in[1,\infty[$ such that, for each real sequence $(y_k)_{k\in\mathbb N}$ and $y_0\in[0,d[$ satisfying:
\begin{align*}
y_k\leq\theta\,y_{k-1}+(1-\theta)\,d\,\exp(-\nu(k-1)T),\qquad k\in\mathbb N,
\end{align*}
then $y_k\leq C(T,\theta,\varrho)\,d\,\exp(-\varrho kT)$ for each $k\in\mathbb N$.
\end{lemma}

\medskip

\begin{proof}
Given that $\varrho\in\,]0,-\log(\theta)/T[\,$, we have $\exp(-\varrho T)\in\,]\theta,1[$. Thus, we define:
\begin{align*}
C(T,\theta,\varrho)=\frac{1-\theta}{\exp(-\varrho T)-\theta}\in[1,\infty[\,.
\end{align*}
For brevity, we denote $C=C(T,\theta,\varrho)$. Let us proceed by induction in $k\in\mathbb N$.

\medskip

\textbf{Base case}. For $k=1$, we have:
\begin{align*}
y_1 & \leq\theta\,y_0+(1-\theta)\,d\\
& \leq (\theta+C(\exp(-\varrho T)-\theta))\,d\\
& \leq C\,d\,\exp(-\varrho T).
\end{align*}

\medskip

\textbf{Induction hypothesis}. For $k\in\mathbb N$ with $k\geq 2$, we have $y_{k-1}\leq C\,d\,\exp(-\varrho(k-1)T)$.

\medskip

\textbf{Inductive step}. 
Note that $\exp((\varrho-\nu)(k-1)T)\in\,]0,1[$ since $\varrho\in\,]0,\nu[\,$. Hence:
\begin{align*}
\exp(-\nu(k-1)T+\varrho kT) & =\exp(\varrho T+(\varrho-\nu)(k-1)T)\\
& <\exp(\varrho T). 
\end{align*}
From this, the recurrence inequality, the induction hypothesis and the definition of $C=C(T,\theta,\varrho)$, we conclude:
\begin{align*}
y_k & \le \theta\,y_{k-1} + (1-\theta)\,d\,\exp(-\nu(k-1)T) \\
& \le\theta\,C\,d\,\exp(-\varrho (k-1)T) + (1-\theta)\,d\,\exp(-\nu(k-1)T)\\
& =\left(\theta\,C\,\exp(\varrho T) + (1-\theta)\exp(\varrho T)\right)d\exp(-\varrho kT)\\
& =C\,d\,\exp(-\varrho kT).
\end{align*}
\end{proof}

\medskip

Now we prove Claim~\ref{PE:claim}, which was used in the proof of Theorem~\ref{theorem:localavg}.

\medskip

\begin{proof}[of Claim~\ref{PE:claim}]
Let us show each statement: 
\begin{enumerate}
    \item By contradiction, suppose that $t_0=\inf\{t\in[(k-1)T,kT[\,\mid V_i(t)=Rd_i\}\in[(k-1)T,kT[$. Given that $Rd_i<d_\varepsilon$ by construction, we have $V_i(t)\leq d_\varepsilon$ for each $t\in[(k-1)T,t_0]$. Hence, from Proposition~\ref{prop:windowed-ISS} and~\eqref{proof:exponential}, we obtain:
    \begin{align}\nonumber
    \dot V_i & \leq-(a_c(t,z_i)-\varepsilon)\,V_i+M\sum_{j=1}^{i-1}|e_j|\\\nonumber
    & \leq-(a_c(t,z_i)-\varepsilon)\,V_i+M\sum_{j=1}^{i-1}\sqrt{2 M_j d_j}\exp\left(-\frac{\rho_j t}{2}\right)\\\label{proof:diffineq}
    & \leq-\eta_\varepsilon\,V_i+M\sum_{j=1}^{i-1}\sqrt{2 M_j d_j},
    \end{align}
    for each $t\in[(k-1)T,t_0]$, where we used that $a_c(t,z_i)-\varepsilon\geq\eta_\varepsilon$ by definition. For brevity, we write $\tau=t_0-(k-1)T\in[0,T[\,$. Moreover, recall that $V_i((k-1)T)<d_i$, $r=\exp(-\eta_\varepsilon T)$ and $H=T\,M\exp(-\eta_\varepsilon T)$.
    By integrating~\eqref{proof:diffineq} between $t=(k-1)T$ and $t=t_0$ and using Lemma~\ref{lemma:cisavg}, we arrive at a contradiction. Indeed, we denote $K=M\sum_{j=1}^{i-1}\sqrt{2M_jd_j}$ and we distinguish two cases:
    \begin{enumerate}
        \item $\eta_\varepsilon=0$:
        \begin{align*}
        Rd_i=V_i(t_0) & \leq V_i((k-1)T)+\tau\,K\\
        & <d_i+T\,K<Rd_i.
        \end{align*}
        \item $\eta_\varepsilon\in\mathbb R^-$:
        \begin{align*}
        \dot V_i\leq-\eta_\varepsilon V_i+K,
        \end{align*}
        for each $t\in[(k-1)T,t_0]$. By using the integrating factor $\exp(\eta_\varepsilon t)$, we obtain:
        \begin{align*}
        \frac{d}{dt}\left(\exp(\eta_\varepsilon t)V_i\right)\leq K\exp(\eta_\varepsilon t),
        \end{align*}
        for each $t\in[(k-1)T,t_0]$. By writing $\tau=t_0-(k-1)T\in[0,T[\,$, we obtain:
        \begin{align*}
        & R d_i =V_i(t_0)\\ 
        & \leq\exp(-\eta_\varepsilon\tau)\,V_i((k-1)T)+\frac{K}{\eta_\varepsilon}(1-\exp(-\eta_\varepsilon\tau))\\
        & =\exp(-\eta_\varepsilon\tau)\,V_i((k-1)T)+K\int_0^\tau\exp(-\eta_\varepsilon s)\,ds\\
        & \leq\exp(-\eta_\varepsilon T)(d_i+TM\sum_{j=1}^{i-1}\sqrt{2M_jd_j})<Rd_i.
        \end{align*}
    \end{enumerate}
     Hence, $V_i(t)\in[0,Rd_i[$ for each $t\in[(k-1)T,kT[\,$.

    \medskip
    
    \item Given that $d_i\in\,]0,d_\varepsilon/R[\,$, from the previous item we conclude that $V_i(t)\in[0,d_\varepsilon[$ for each $t\in[(k-1)T,kT[\,$. Moreover, $d_\varepsilon\in\,]0,1-c]$, whence $V_i(t)\in[0,1-c[\,$, \textit{i.e.}, $z_i(t)\in\Omega_c^{\sigma_i}(\tilde\mu)$, for each $t\in[(k-1)T,kT[\,$. Thus, $a_c(t,z_i(t))\geq\alpha_c(t)$ for each $t\in[(k-1)T,kT[\,$. Thus, Assumption~\ref{ass:selfavg} gives:
    \begin{align}\label{proof:PEpart2}
    \int_{(k-1)T}^{kT}(a_c(s,z_i(s))-\varepsilon)\,ds\geq(\rho-\varepsilon)T.
    \end{align}
    In addition, Proposition~\ref{prop:windowed-ISS} holds, yielding:
    \begin{align*}
    \dot V_i & \leq-(a_c(t,z_i)-\varepsilon)\,V_i+M\sum_{j=1}^{i-1}|e_j|\\
    & \leq-(\alpha_c(t)-\varepsilon)\,V_i+\sum_{j=1}^{i-1}K_j\exp\left(-\frac{\rho_j}{2}t\right),
    \end{align*}
    for each $t\in[(k-1)T,kT[\,$, where we used~\eqref{proof:exponential} and denoted $K_j=M\sqrt{2M_jd_j}$ for brevity. The integrating factor $\phi(t)=\exp\left(\int_{(k-1)T}^t(\alpha_c(s)-\varepsilon)\,ds\right)$ gives:
    \begin{align}\nonumber
    \frac{d}{dt}(\phi(t)V_i) & \leq \sum_{j=1}^{i-1}K_j\exp\left(-\frac{\rho_j}{2}t\right)\phi(t)\\\label{proof:integrating}
    & \leq \sum_{j=1}^{i-1}K_j\exp\left(-\frac{\rho_j}{2}(k-1)T\right)\phi(t),
    \end{align}
    for each $t\in[(k-1)T,kT[\,$. Note that:
    \begin{align*}
    \phi(kT)^{-1}&=\exp\left(-\int_{(k-1)T}^{kT}(\alpha_c(s)-\varepsilon)\,ds\right)\\&\leq\exp(-(\rho-\varepsilon)T)=\theta,
    \end{align*}
    where we used~\eqref{proof:PEpart2}. Thus, by integrating~\eqref{proof:integrating} between $t=(k-1)T$ and $t=kT$, we obtain:
    \begin{align*}
    V_i(kT) & \leq V_i((k-1)T)\phi(kT)^{-1}\\
    &+K\int_{(k-1)T}^{kT}\exp\left(-\int_t^{kT}(\alpha_c(s)-\varepsilon)\,ds\right)dt\\
    & \leq \theta\,V_i((k-1)T)+ KT\exp(-\eta_\varepsilon\,T),
    \end{align*}
    where we used that $\alpha_c(t)-\varepsilon\geq\eta_\varepsilon$ and $\phi((k-1)T)=1$. 
    
    \medskip
    
    \item From the previous item, as well as $V_i((k-1)T)\in[0,d_i[$ and Lemma~\ref{lemma:cisavg}, we conclude:
    \begin{align*}
    V_i(kT)\leq\theta d_i+H\sum_{j=1}^{i-1}\sqrt{2M_jd_j}<d_i,
    \end{align*}
    where we recall that $H=TM\,\exp(-\eta_\varepsilon T)$.
\end{enumerate}
\end{proof}


Now we write the proofs for section~\ref{sec:global}. We begin by showing the bounds for the projected variables.

\begin{proof}[of Proposition~\ref{prop:dotbibound}]
Firstly, note that $]c_\textrm{eff},b_i^\textrm{max}[\,\neq\emptyset$ as $Z_0\not\in\mathcal B_\ell(\tilde\mu)$. Note that $V_i|_{]c_i,1[}>0$ and $V_i(1)=0$. Hence, its dynamics readily follows from~\eqref{PE:dotbi1},~\eqref{PE:dotbi2},~\eqref{PE:dotbi3},~\eqref{PE:dotbi4} and~\eqref{PE:dotbi5}. Using that $\dot V_i=-\dot b_i$, let us bound each term separately on $I_{c,i}^\Sigma(Z_0,\tilde\mu)$.

For the first term, we have (for brevity, we drop the argument $t$):
\begin{align}\label{dotbi:1}
-\eqref{PE:dotbi1} & \leq-\alpha\,\gamma\,c\,(1-b_i^2),
\end{align}
where we used~\eqref{proof:gap} and $\alpha=\inf_{(t,z)\in\mathbb R_0^+\times\mathbb S^{n-1}}z^\top\,S_{BC}(t)\,z$.

For the second term, by writing $z_i=b_i\,\mathfrak{e}^{\tilde\mu}+\sqrt{1-b_i^2}\,\epsilon_{\tilde\mu}$, with $\epsilon_{\tilde\mu}\in\mathbb S^{n-1}$ such that $\epsilon_{\tilde\mu}^\top\,\mathfrak{e}^{\tilde\mu}=0$, we have $z_i^\top S_{BC}\,\sigma_i\,\mathfrak{e}^{\tilde\mu}=b_i\left((\mathfrak{e}^{\tilde\mu})^\top\,S_{BC}\,\mathfrak{e}^{\tilde\mu}\right)+r_i(t)$, where:
\begin{align*}
r_i=\sqrt{1-b_i^2}\,\epsilon_{\tilde\mu}^\top\,(S_{BC}-((\mathfrak{e}^{\tilde\mu})^\top S_{BC}\,\mathfrak{e}^{\tilde\mu})\mathbb I_n)\,\mathfrak{e}^{\tilde\mu}.
\end{align*}
Note that $|r_i|\leq\alpha\,c_\star\,\sqrt{1-b_i^2}$. Hence, we have:
\begin{align}\nonumber
-\eqref{PE:dotbi2} & =-\sum_{j=1}^{i-1}\sigma_i\,\lambda_{ij}^{\tilde\mu}\,(z_i^\top S_{BC}\,\mathfrak{e}^{\tilde\mu})(1-b_i^2)\\\nonumber
& =-\sum_{j=1}^{i-1}\lambda_{ij}^{\tilde\mu}\left(b_i\left((\mathfrak{e}^{\tilde\mu})^\top\,S_{BC}\,\mathfrak{e}^{\tilde\mu}\right)+r_i(t)\right)(1-b_i^2)\\\label{dotbi:2}
& \leq-(i-1)\,\alpha\left(\lambda_{\min}\,c-\lambda_{\max}\,c_\star\sqrt{1-c^2}\right)(1-b_i^2).
\end{align}

From \eqref{dotbi:1} and \eqref{dotbi:2}, we obtain:
\begin{align}\nonumber
-\eqref{PE:dotbi1}-\eqref{PE:dotbi2}  \leq&-\alpha\Big(\gamma\,c+(i-1)\,\lambda_{\min}\,c\\&-(i-1)\,\lambda_{\max}\,c_\star\sqrt{1-c^2}\Big)(1-b_i^2)\nonumber\\\label{dotbi} 
 \leq&-\alpha\,\beta\,(1+c)\,V_i,
\end{align}
where $\beta=c\,(\gamma+(i-1)\,\lambda_{\min})-c_\star\,\lambda_{\max}\,(i-1)\,\sqrt{1-c^2}$. Note that $\beta\in\mathbb R^+$. Indeed, by squaring:
\begin{align*}
c\,(\gamma+(i-1)\,\lambda_{\min})>c_\star\,\lambda_{\max}\,(i-1)\,\sqrt{1-c^2}
\end{align*}
and rearranging terms, we obtain:
\begin{align*}
& c^2\left((\gamma+(\ell-1)\,\lambda_{\min})^2+c_\star^2\,\lambda_{\max}^2\,(\ell-1)^2\right)\\
& \quad \geq c^2\left((\gamma+(i-1)\,\lambda_{\min})^2+c_\star^2\,\lambda_{\max}^2\,(i-1)^2\right)\\
& \quad >c_\star^2\,\lambda_{\max}^2\,(i-1)^2\geq c_\star^2\,\lambda_{\max}^2.
\end{align*}
This holds provided:
\begin{equation*}
c>\frac{c_\star\,\lambda_{\max}}{\sqrt{(\gamma+(\ell-1)\,\lambda_{\min})^2+c_\star^2\,\lambda_{\max}^2\,(\ell-1)^2}}=c_\textrm{eff}.
\end{equation*}

Lastly, from Assumption~\ref{ass:SSM} and Lemma~\ref{lemma:lambdabound}, as well as compactness of the sphere, the error terms are all bounded:
\begin{align}\label{dotbi:3}
    -\eqref{PE:dotbi3}-\eqref{PE:dotbi4}-\eqref{PE:dotbi5}\leq M\sum_{j=1}^{i-1}|e_j|.
\end{align}

By gathering~\eqref{dotbi} and~\eqref{dotbi:3}, and denoting $\kappa=\alpha\,\beta\,(1+c)\in\mathbb R^+$, we conclude:
\begin{align*}
\dot V_i & \leq-\kappa\,V_i+M\sum_{j=1}^{i-1}|e_j|.
\end{align*}

\end{proof}

\medskip

Lastly, the next result is readily obtained by an ISS argument. It guarantees convergence to consensus of the last token at the inductive step in the proof of Theorem~\ref{theorem:global}.

\medskip

\begin{lemma}\label{lemma:comparison}
Let $b : \mathbb{R}_0^+ \to [-1,1]$ be continuously differentiable and $e: \mathbb{R}_0^+ \to \mathbb{R}_0^+$ be continuous with $\lim_{t \to \infty} e(t) = 0$. If there exist $\kappa \in \mathbb{R}^+$ and $t_0\in\mathbb R_0^+$ such that $\dot{b}\geq \kappa\,(1 - b) - e(t)$ for each $t\in[t_0,\infty[\,$, then $\lim_{t \to \infty} b(t) = 1$.
\end{lemma}

\medskip

\end{document}